\documentclass{article}
\usepackage[utf8]{inputenc}

\usepackage{microtype}
\usepackage{graphicx}
\usepackage{subfigure}
\usepackage{booktabs} 
\usepackage{kpfonts}
\usepackage{fullpage}

\usepackage[numbers]{natbib}
\usepackage{hyperref}

\usepackage{amsmath}
\usepackage{amssymb}

\usepackage[ruled]{algorithm2e}
\usepackage{setspace}
\usepackage{amsthm}
\usepackage{cleveref}

\usepackage{chngcntr}
\usepackage{apptools}
\usepackage[normalem]{ulem}

\newtheorem{theorem}{Theorem}
\newtheorem*{theorem*}{Theorem}
\newtheorem{lemma}[theorem]{Lemma}
\newtheorem*{lemma*}{Lemma}
\newtheorem*{claim*}{Claim}

\usepackage{bbm}
\newtheorem*{remark}{Remark}
\newtheorem{claim}{Claim}
\theoremstyle{definition}
\newtheorem{definition}{Definition}
\usepackage{graphicx}
\graphicspath{ {../images/} }
\usepackage{xcolor}
\usepackage{bbm}
\usepackage{hyperref}
\usepackage{mathtools}
\usepackage{breqn}
\usepackage{xcolor}

\usepackage{tikzit}

\tikzstyle{States}=[fill=white, draw=black, shape=circle, tikzit draw=black]

\tikzstyle{arrow}=[->, tikzit draw=black]

\DeclareMathOperator*{\argmax}{arg\,max}
\DeclareMathOperator*{\conf}{\mathrm{conf}}

\newcommand{\cA}{\mathcal{A}}
\newcommand{\cB}{\mathcal{B}}
\newcommand{\cS}{\mathcal{S}}
\newcommand{\cP}{\mathcal{P}}

\newcommand{\cM}{\mathcal{M}}

\newcommand{\cR}{\mathcal{R}}

\newcommand{\algo}{\cM}

\newcommand{\eps}{\varepsilon}
\renewcommand{\epsilon}{\varepsilon}

\newcommand{\bE}{\mathbb{E}}

\newcommand{\Ex}[1]{\bE \left[ #1 \right]}
\newcommand{\ex}[2]{\bE_{#1} \left[ #2 \right]}

\newcommand{\pr}[1]{\textrm{Pr}\left[ #1 \right]}
\newcommand{\prob}[2]{\textrm{Pr}_{#1}\left[ #2 \right]}

\usepackage[textsize=footnotesize,color=green!40]{todonotes}

\newcommand{\actions}{\cA}
\newcommand{\states}{\cS}
\newcommand{\tranKer}{\cP}
\newcommand{\reDist}{\cR}
\newcommand{\bernoulli}{B}

\newcommand{\polylog}{\mathrm{polylog}}
\newcommand{\pit}{\pi_t}

\newcommand{\ntil}{\widetilde{n}}
\newcommand{\ntilt}{\widetilde{n}_t}
\newcommand{\mtil}{\widetilde{m}}
\newcommand{\mtilt}{\widetilde{m}_t}
\newcommand{\rtil}{\widetilde{r}}
\newcommand{\rtilt}{\widetilde{r}_t}
\newcommand{\nhat}{\widehat{n}}
\newcommand{\mhat}{\widehat{m}}
\newcommand{\rhat}{\widehat{r}}
\newcommand{\nhatt}{\widehat{n}_t}
\newcommand{\mhatt}{\widehat{m}_t}
\newcommand{\rhatt}{\widehat{r}_t}
\newcommand{\Vtil}{\widetilde{V}}
\newcommand{\Qtil}{\widetilde{Q}}
\newcommand{\QtilUB}{\widetilde{Q}^{+}}
\newcommand{\Qhat}{\widehat{Q}}
\newcommand{\QhatUB}{\widehat{Q}^{+}}
\newcommand{\conftil}{\widetilde{\conf}}
\newcommand{\conftilt}{\widetilde{\conf}_t}

\newcommand{\BMname}{\text{PC}  }
\newcommand{\BM}[1]{\BMname \left( #1 \right) }

\newcommand{\Eeps}{E_{\eps}}

\newcommand{\sah}{\text{x}}

\newcommand{\Tmax}{T}
\newcommand{\counterBound}{ \frac{3}{\eps}H\log\left(\frac{2SAH+S^2AH}{\beta'}\right)\log\left(\Tmax\right)^{5/2}}
\newcommand{\phihatconfL}[1]{\widehat{\phi}_t(#1)}
\newcommand{\phihatconfR}[1]{\sqrt{\frac{\Tmax/\beta'}{\nhatt(#1)}}}
\newcommand{\phitilconfL}[1]{\widetilde{\phi}_t(#1)}
\newcommand{\phitilconfR}[1]{\sqrt{\frac{2\ln{(\Tmax/\beta')}}{\max(\ntilt(#1)-\Eeps, 1)}}}
\newcommand{\psitilconfL}[1]{\widetilde{\psi}_t(#1)}
\newcommand{\psitilconfR}[1]{(1+SH)\left(  \frac{3\Eeps}{\ntilt(#1)}+ \frac{2\Eeps^2}{\ntilt(#1)^2} \right)}

\newcommand{\eqdefU}{\coloneqq}

\newcommand{\PUCB}{\texttt{PUCB}}
\newcommand{\PrivQ}{\texttt{PrivQ}}
\newcommand{\regret}{\mathrm{Regret}}

\title{Private Reinforcement Learning with PAC and Regret Guarantees}
\author{Giuseppe Vietri\thanks{Department of Computer Science and Engineering, University of Minnesota.
Supported by the GAANN fellowship from the U.S. Department of Education. Email: \href{mailto:vietr002@umn.edu}{vietr002@umn.edu}}
\and
Borja Balle\thanks{Now at DeepMind. Email: \href{mailto:borja.balle@gmail.com}{borja.balle@gmail.com}}
\and
Akshay Krishnamurthy\thanks{Microsoft Research, New York, NY. Email: \href{mailto:akshaykr@microsoft.com}{akshaykr@microsoft.com}}
\and
Zhiwei Steven Wu\thanks{School of Computer Science, Carnegie Mellon University. Email: \href{mailto:zstevenwu@cmu.edu}{zstevenwu@cmu.edu}}
}
\date{}

\allowdisplaybreaks

\begin{document}

 \maketitle

\begin{abstract}
Motivated by high-stakes decision-making domains like personalized
medicine where user information is inherently sensitive, we design
privacy preserving exploration policies for episodic reinforcement
learning (RL). We first provide a meaningful privacy formulation using
the notion of joint differential privacy (JDP)--a strong variant of
differential privacy for settings where each user receives their own
sets of output (e.g., policy recommendations). We then develop a
private optimism-based learning algorithm that simultaneously achieves
strong PAC and regret bounds, and enjoys a JDP guarantee. Our
algorithm only pays for a moderate privacy cost on exploration: in
comparison to the non-private bounds, the privacy parameter only
appears in lower-order terms.  Finally, we present lower bounds on
sample complexity and regret for reinforcement learning subject to
JDP.
\end{abstract}

\section{Introduction}

Privacy-preserving machine learning is critical to the deployment of data-driven solutions in applications involving sensitive data.
Differential privacy (DP) \cite{dwork2006calibrating} is a de-facto standard for designing algorithms with strong privacy guarantees for individual data.
Large-scale industrial deployments -- e.g.\ by Apple \cite{appledp}, Google \cite{erlingsson2014rappor} and the US Census Bureau \cite{abowd2018us} -- and general purpose DP tools for machine learning \cite{tfprivacy} and data analysis \cite{DBLP:journals/corr/abs-1907-02444,wilson2019differentially} exemplify that existing methods are well-suited for simple data analysis tasks (e.g.\ averages, histograms, frequent items) and batch learning problems where the training data is available beforehand.
While these techniques cover a large number of applications in the central and (non-interactive) local models, they are often insufficient to tackle machine learning applications involving other threat models.
This includes federated learning problems \cite{kairouz2019advances,li2019federated} where devices cooperate to learn a joint model while preserving their individual privacy, and, more generally, interactive learning in the spirit of the reinforcement learning (RL) framework \cite{sutton2018reinforcement}.

In this paper we contribute to the study of reinforcement learning from the lens of differential privacy.
We consider sequential decision-making tasks where users interact with an agent for the duration of a fixed-length episode. At each time-step the current user reveals a state to the agent, which responds with an appropriate action and receives a reward generated by the user.
Like in standard RL, the goal of the agent is to learn a policy that maximizes the rewards provided by the users.
However, our focus is on situations where the states and rewards that users provide to the agent might contain sensitive information.
While users might be ready to reveal such information to an agent in order to receive a service, we assume they want to prevent third parties from making unintended inferences about their personal data.
This includes external parties who might have access to the policy learned by the agent, as well as malicious users who can probe the agent's behavior to trigger actions informed by its interactions with previous users.
For example, \cite{DBLP:conf/atal/PanWZLYS19} recently showed how RL policies can be probed to reveal information about the environment where the agent was trained.

The question we ask in this paper is: how should the learnings an agent can extract from an episode be balanced against the potential information leakages arising from the behaviors of the agent that are informed by such learnings?
We answer the question by making two contributions to the analysis of the privacy-utility trade-off in reinforcement learning: (1) we provide the first privacy-preserving RL algorithm with formal accuracy guarantees, and (2) we provide lower bounds on the regret and number of sub-optimal episodes for any differentially private RL algorithm.
To measure the privacy provided by episodic RL algorithms we introduce a notion of episodic joint differential privacy (JDP) under continuous observation, a variant of joint differential privacy \cite{DBLP:conf/innovations/KearnsPRU14} that captures the potential information leakages discussed above.

\paragraph{Overview of our results.}
We study reinforcement learning in a fixed-horizon episodic Markov
decision process with $S$ states, $A$ actions, and episodes of length
$H$.  We first provide a meaningful privacy formulation for this
general learning problem with a strong relaxation of differential
privacy: joint differential privacy (JDP) under continual
observation, controlled by a privacy parameter $\epsilon \geq 0$ (larger $\epsilon$ means less privacy). Under this formulation, we give the first known RL
sample complexity and regret upper and lower bounds with formal
privacy guarantees.  First, we present a new algorithm, \PUCB, which
satisfies $\eps$-JDP in addition to two utility guarantees: it finds
an $\alpha$-optimal policy with a sample complexity of
$$\tilde O\left( \frac{SAH^4}{\alpha^2} + \frac{S^2AH^4
  }{\eps\alpha}\right) \enspace,$$ and achieves a regret rate of
$$\tilde O\left(H^2 \sqrt{SAT} + \frac{SAH^3 + S^2 A H^3}{\eps}
\right)$$ over $T$ episodes. In both of these bounds, the first terms
$\frac{SAH^4}{\alpha^2}$ and $H^2 \sqrt{SAT}$ are the non-private sample
complexity and regret rates, respectively. The privacy parameter $\eps$ only affects the lower order
terms -- for sufficiently small approximation $\alpha$ and sufficiently
large $T$, the ``cost'' of privacy becomes negligible.

We also provide new lower bounds for $\eps$-JDP reinforcement
learning. Specifically, by incorporating ideas from existing lower bounds for
private learning into constructions of hard MDPs, we prove a sample
complexity bound of
$$\tilde \Omega\left(\frac{SAH^2}{\alpha^2} + \frac{SAH}{\eps\alpha} \right)$$
 and a regret bound of
$$\tilde \Omega\left( \sqrt{HSAT} + \frac{SAH}{\eps}\right) \enspace.$$
 As expected, these lower bounds match our upper bounds in the dominant term (ignoring $H$
and polylogarithmic factors). We also see that necessarily the utility
cost for privacy grows linearly with the state space size, although
this does not match our upper bounds. Closing this gap is an
important direction for future work.



\subsection{Related Work}

Most previous works on differentially private interactive learning with partial feedback concentrate on bandit-type problems, including on-line learning with bandit feedback \cite{thakurta2013nearly,agarwal2017price}, multi-armed bandits \cite{mishra2015nearly,tossou2016algorithms,tossou2017achieving,tossou2018differential}, and linear contextual bandits \cite{DBLP:conf/icml/NeelR18,shariff2018differentially}.
These works generally differ on the assumed reward models under which utility is measured (e.g.\ stochastic, oblivious adversarial, adaptive adversarial) and the concrete privacy definition being used (e.g.\ privacy when observing individual actions or sequences of actions, and privacy of reward or reward and observation in the contextual setting).
\cite{basu2019differential} provides a comprehensive account of different privacy definitions used in the bandit literature.


Much less work has addressed DP for general RL.
For policy evaluation in the batch case, \cite{DBLP:conf/icml/BalleGP16} propose regularized least-squares algorithms with output perturbation and bound the excess risk due to the privacy constraints.
For the control problem with private rewards and public states, \cite{NIPS2019_9310} give a differentially private Q-learning algorithm with function approximation.

On the RL side, as we are initiating the study of RL with differential
privacy, we focus on the well-studied tabular setting.  While a number
of algorithms with utility guarantees and lower bound constructions
are known for this setting~\cite{Kakade2003,azar2017minimax,dann2017unifying}, we are not aware of any work
addressing the privacy issues that are fundamental in high-stakes
applications.



\section{Preliminaries}

\subsection{Markov Decision Processes}

A fixed-horizon \emph{Markov decision process} (MDP) with
time-dependent dynamics can be formalized as a tuple
$M = \left( \states, \actions, \reDist, \tranKer, p_0, H
\right)$. $\states$ is the state space with cardinality
$S$. $\actions$ is the action space with cardinality $A$. $\reDist(s_h, a_h, h)$ is
the reward distribution on the interval $[0,1]$ with mean
$r(s_h, a_h, h)$. $\tranKer$ is the transition kernel, given time step
$h$, action $a_h$ and, state $s_h$ the next state is sampled from
$s_{t+1} \sim \tranKer(.|s_h, a_h,h)$. Let $p_0$ be the initial state
distribution at the start of each episode, and $H$ be the number of
time steps in an episode.

In our setting, an agent interacts with an MDP by following a
(deterministic) policy $\pi \in \Pi$, which maps states $s$ and
timestamps $h$ to actions, i.e., $\pi(s,h) \in \actions$.
The \emph{value function} in time step $h\in [H]$ for a policy $\pi$ is defined as:
\begin{align*}
  V_h^\pi(s) &\;= \Ex{\sum_{i=h}^H r(s_i, a_i, i) \bigg| s_h  = s, \pi} \\
   &\;= r(s, \pi(s,h), h) + \sum_{s'\in \states} V_{h+1}^\pi(s') \tranKer(s'|s,\pi(s,h), h)
   \enspace.
\end{align*}
The \emph{expected total reward} for policy $\pi$ during an entire episode is:
$$\rho^\pi = \Ex{\sum_{i=1}^H r(s_i, a_i, i) \bigg| \pi} = p_0^\top V_1^\pi \enspace.$$
The \emph{optimal value function} is given by $V_h^*(s) = \max_{\pi \in \Pi} V_h^{\pi}(s)$.
Any policy $\pi$ such that $V_h^{\pi}(s) = V_h^*(s)$ for all $s \in \states$ and $h \in [H]$ is called optimal. It achieves the optimal expected total reward $\rho^* = \max_{\pi \in \Pi} \rho^{\pi}$.

The goal of an RL agent is to learn a near-optimal policy after interacting with an MDP for a finite number of episodes $\Tmax$.
During each episode $t \in [\Tmax]$ the agent follows a policy $\pit$ informed by previous interactions, and after the last episode it outputs a final policy $\pi$.

\begin{definition}
  An agent is \emph{$(\alpha,\beta)$-probably approximately correct} (PAC) with sample complexity $f(S,A,H,\tfrac{1}{\alpha}, \log(\tfrac{1}{\beta}))$, if with probability at least $1-\beta$
it follows an $\alpha$-optimal policy $\pi$ such that
$\rho^* - \rho^\pi \leq \alpha$ except for at most
$f(S,A,H,\tfrac{1}{\alpha}, \log(\tfrac{1}{\beta}))$ episodes.
\end{definition}

\begin{definition}
The (expected cumulative) \emph{regret} of an agent after $T$ episodes is given by
\begin{align*}
  \regret(T) = \sum_{t=1}^T (\rho^* - \rho^{\pi_t})
  \enspace,
\end{align*}
where
$\pi_1, \ldots \pi_\Tmax$ are the policies followed by the agent
on each episode.
\end{definition}

\subsection{Privacy in RL}
In some RL application domains such as personalized medical treatments,
the sequence of states and rewards received by a reinforcement
learning agent may contain sensitive information. For example,
individual users may interact with an RL agent for the duration of an
episode and reveal sensitive information in order to obtain a service
from the agent.  This information affects the final policy produced by
the agent, as well as the actions taken by the agent in any subsequent
interaction.  Our goal is to prevent damaging inferences about a
user's sensitive information in the context of the interactive
protocol in \cref{alg:rlprotocol} summarizing the
interactions between an RL agent $\algo$ and $\Tmax$ distinct users.

\begin{algorithm}[h]
\caption{Episodic RL Protocol}\label{alg:rlprotocol}
\KwIn{ Agent $\algo$ and users $u_1, \ldots, u_\Tmax$}
\For{$t \in [\Tmax]$}{
  \For{$h \in [H]$}{
    $u_t$ sends state $s_h^{(t)}$ to $\algo$ \\
    $\algo$ sends action $a_h^{(t)}$ to $u_t$ \\
    $u_t$ sends reward $r_h^{(t)}$ to $\algo$ \\
  }
}
$\algo$ releases policy $\pi$
\end{algorithm}

Throughout the execution of this protocol the agent observes a collection of $\Tmax$ state-reward trajectories of length $H$. Each user $u_t$ gets to observe the actions chosen by the agent during the $t$-th episode, as well as the final policy $\pi$.
To preserve the privacy of individual users we enforce a (joint) differential privacy criterion: upon changing one of the users in the protocol, the information observed by the other $\Tmax-1$ participants will not change substantially.
This criterion must hold even if the $\Tmax-1$ participants collude adversarially, by e.g., crafting their states and rewards to induce the agent to reveal information about the remaining user.

Formally, we write $U = (u_1,\ldots,u_\Tmax)$ to denote a sequence of $\Tmax$ users participating in the RL protocol.
Technically speaking a user can be identified with a tree of depth $H$ encoding the state and reward responses they would give to all the $A^H$ possible sequences of actions the agent can choose. During the protocol the agent only gets to observe the information along a single root-to-leaf path in each user's tree.
For any $t \in [\Tmax]$, we write $\algo_{-t}(U)$ to denote all the outputs excluding
the output for episode $t$ during the interaction between $\algo$ and $U$. This captures all the outputs which might leak information about the $t$-th user in interactions after the $t$-th episode, as well as all the outputs from earlier episodes where other users could be submitting information to the agent adversarially to condition its interaction with the $t$-th users.

We also say that two user sequences $U$ and $U'$ are $t$-neighbors if they only differ in their $t$-th user.

\begin{definition}\label{def:jdp}
A randomized RL agent $\algo$ is \emph{$\epsilon$-jointly differentially private under continual observation} (JDP) if for all $t \in [\Tmax]$, all $t$-neighboring user sequences $U$, $U'$, and all events
$E \subseteq \actions^{H \times [\Tmax-1]} \times \Pi$
we have
\begin{align*}
\pr{\algo_{-t}(U) \in E }\leq e^\epsilon \pr{\algo_{-t}(U') \in E} \enspace.
\end{align*}
\end{definition}

This definition extends to the RL setting the one used in \cite{shariff2018differentially} for designing privacy-preserving algorithms for linear contextual bandits.
The key distinctions is that in our definition each user interacts with the agent for $H$ time-steps (in bandit problems one usually has $H=1$), and we also allow the agent to release the learned policy at the end of the learning process.

Another distinction is that our definition holds for all past and future outputs. In contrast, the definition of JDP in \cite{shariff2018differentially} only captures future episodes; hence, it only protects against collusion from future users.

To demonstrate that our definition gives a stronger privacy protection, we use a simple example. Consider an online process that takes as input a stream of binary bits $u=(u_1, \ldots, u_T)$, where $u_t\in\{0,1\}$ is the data of user $t$, and on each round $t$ the mechanism outputs the partial sum $m_t(u) = \sum_{i=1}^t u_i$.
Then the following trivial mechanism satisfies JDP  (in terms of future episodes as in the JDP definition of \cite{shariff2018differentially}):
First, sample once from the Laplace mechanism $\xi \sim \text{Lap}(\eps)$ before the rounds begin, and on each round output $\widetilde{m}_t(u) = m_t(u) + \xi$.
Note that the view of any future user $t’>t$ is $\widetilde{m}_{t'}(u)$.
Now let $u$ be a binary stream with user $t$ bit on and let $w$ be identical to $u$ but with user $t$ bit off. Then, by the differential-privacy guarantee of the Laplace mechanism, a user $t'>t$ cannot distinguish between $\widetilde{m}_{t'}(u)$ and $\widetilde{m}_{t'}(w)$. Furthermore, any coalition of future users cannot provide more information about user $t$. Therefore this simple mechanism satisfies the JDP definition from \cite{shariff2018differentially}.

However, the simple counting mechanism with one round of Laplace noise does not satisfy JDP for past and future outputs as in our JDP  (\cref{def:jdp}). To see why, suppose that user $t-1$  and user $t+1$ collude in the following way: For input $u$, the view of user $t-1$ is $\widetilde{m}_{t-1}(u)$ and the view of user $t+1$ is $\widetilde{m}_{t+1}(u)$. They also know their own data $u_{t-1}$, $u_{t+1}$. Then they can recover the data of the $t$-th user as follows
\begin{align*}
    \widetilde{m}_{t+1}(u) - u_{t+1} - \widetilde{m}_{t-1}(u)
= m_{t+1}(u) + \xi - u_{t+1}  - m_{t-1}(u) - \xi
= \sum_{i=1}^{t+1}u_i - u_{t+1}  - \sum_{i=1}^{t-1}u_i
=u_t
\end{align*}

\begin{remark}
1. would the algorithm leak more info for the returning user? yes, but we could bound using group privacy.
2. would other users be affected? no, because JDP prevents arbitrary collusion
\end{remark}

\subsection{Counting Mechanism}

The algorithm we describe in the next section maintains a set of counters
to keep track of events that occur when interacting with
the MDP. We denote by $\nhatt(s,a,h)$ the count of visits to state tuple
$(s,a,h)$ right before episode $t$, where
$a\in\actions$ is the action taken on state $s\in\states$ and
time-step $h\in [H]$. Likewise $\mhatt(s,a,s',h)$ is the count of going
from state $s$ to $s'$ after taking actions $a$ before episode $t$.
Finally, we have the counter $\rhatt(s,a,h)$ for the total reward received
by taking action $a$ on state $s$ and time $h$ before episode $t$.
Then, on episode $t$, the counters are sufficient
to create an estimate of the MDP dynamics to construct a policy for episode $t$. The challenge is that
the counters depend on the sequence of states and actions, which is considered sensitive data in this work.
Therefore the algorithm must release the counts in a privacy-preserving
way, and we do this the private counters proposed by
\cite{chan2011private} and \cite{dwork2010differential}.

A private counter mechanism takes as input a stream
$\sigma = (\sigma_1\ldots, \sigma_T)\in [0,1]^T$ and on any round $t$
releases and approximation of the prefix count
$c(\sigma)(t) = \sum_{i=1}^t \sigma_i$.
In this work we will denote $\BMname$ as the binary mechanism of
\cite{chan2011private} and \cite{dwork2010differential} with parameters $\epsilon$ and $T$. This
mechanism produces a monotonically increasing count and satisfies the
following accuracy guarantee: Let $\cM \coloneqq \BM{T, \eps}$ be a
private counter and $c(\sigma)(t)$ be the true count on episode
$t$, then given a stream $\sigma$, with probability at least
$1-\beta$, simultaneously for all $1\leq t \leq T$, we have
\begin{align*}
\left| \cM(\sigma)(t) - c(\sigma)(t) \right| \leq \frac{4}{\eps}\ln(1/\beta)\log(T)^{5/2}
\enspace.
\end{align*}

While the stated bound above holds for a single $\eps$-DP counter, our
algorithm needs to maintain more than $S^2 A H$ many counters. A naive
allocation of the privacy budget across all these counters will
require noise with scale polynomially with $S, A$, and $H$. However, we
will leverage the fact that the total change across all counters a
user can have scales with the length of the episode $H$, which allows
us to add a much smaller amount of noise that scales linearly in $H$.

\section{The $\PUCB$ Algorithm}
In this section, we introduce the Private Upper Confidence Bound algorithm ($\PUCB$), a JDP algorithm with both PAC and regret
guarantees.  The pseudo-code for $\PUCB$ is in
algorithm~\ref{alg:pucb}.  At a high level, the algorithm is a private version of the \texttt{UBEV} algorithm~\citep{dann2017unifying}. \texttt{UBEV} keeps track of three types of statistics about the history, including (a) the average empirical reward for taking action $a$ in state $s$ at time $h$, denoted $\rhatt(s,a,h)$, (b) the number of times the agent has taken action $a$ in state $s$ at time $h$, denoted $\nhatt(s,a,h)$, and (c) the number of times the agent has taken action $a$ in state $s$ at time
$h$ and transitioned to $s'$, denoted $\mhatt(s,a,s',h)$. In each
episode $t$, \texttt{UBEV} uses these statistics to compute a policy via dynamic programming, executes the policy, and updates the statistics with the observed trajectory.~\cite{dann2017unifying} compute the policy using an optimistic strategy and establish both PAC and regret guarantees for this algorithm.


\SetKwInOut{Parameter}{Parameters}

\begin{algorithm}[t]

\caption{Private Upper Confidence Bound ($\PUCB$)}
\label{alg:pucb}
\Parameter{ Privacy parameter $\epsilon$, target failure probability $\beta$}
\KwIn{Maximum number of episodes $\Tmax$, horizon $H$, state space $\states$, action space $\actions$}
$\eps' \coloneqq \eps / (3H)$\\
\For{$s,a,s',h \in \states\times\actions\times\states\times[H]$}{
	Initialize private counters:
	$\rtil(s,a,h), \ntil(s,a,h), \mtil(s,a,s',h)\coloneqq \BM{\Tmax,\eps',\beta}$\
}
\For{$t \gets 1 $ {\bfseries to} $\Tmax$}{
	{Private planning: }$\QtilUB_t \coloneqq \texttt{{PrivQ}}\left(\rtil, \ntil, \mtil, \eps\right)$\\
	\For{$h \gets 1 $ {\bfseries to} $H$}{
		Let $s$ denote the state during step $h$ and episode $t$\\
		Execute $a \coloneqq \argmax_{a'} \QtilUB_t(s,a',h)$\\
	  Observe $r \sim  \reDist(s, a, h)$ and $s' \sim  \tranKer(.|s,a, h)$\\
		Feed $r$ to $\rtil(s,a,h)$\\
		Feed $1$ to $\ntil(s,a,h)$ and $\mtil(s,a,s',h)$ and $0$ to all other counters $\ntil(\cdot,\cdot,h)$ and $\mtil(\cdot,\cdot,\cdot,h)$\\
	}
}
\end{algorithm}

\begin{algorithm}[t]
\caption{$\PrivQ(\rtil, \ntil,\mtil,\eps,\beta)$}
\label{alg:privateplanning}
\KwIn{  Private counters $\rtil, \ntil, \mtil$, privacy parameter $\eps$, target failure probability $\beta$}
$\Eeps \coloneqq \counterBound$\\
$\widetilde{V}_{H+1}(s) \coloneqq 0$ \ \ $\forall s \in \states$\\
\For{$h \gets H $ {\bfseries to} $1$}{
	\For{$s,a \in  \states \times \actions$}{
		\uIf{$\ntilt(s,a,h) \geq 2\Eeps$}{
			 	$\conftilt(s,a,h) \coloneqq  (H+1)\phitilconfL{s,a,h} + \psitilconfL{s,a,h}$
			}
			\Else{
			 	$\conftilt(s,a,h) \coloneqq H$
			}
			$\Qtil_t(s,a,h) \coloneqq \frac{1}{\ntilt(s,a,h)}\left(	\rtilt(s,a,h) + \sum_{s'\in\states} \Vtil_{h+1} (s') \mtilt(s,a,s',h)\right)$\\
			$\QtilUB_t(s,a, h) \coloneqq \min\left\{H, \Qtil_t(s,a,h) +  \conftilt(s, a, h)\right\}$
	}
	$\Vtil_h(s) \coloneqq \max_a \QtilUB_t (s,a, h)\quad \forall s \in \states$
}
\KwOut{$\QtilUB_t$}
\end{algorithm}

Of course, as the policy depends on the statistics from the previous
episodes, $\texttt{UBEV}$ as is does not satisfy JDP. On the other hand, the policy executed only depends on the previous episodes \emph{only} through the statistics $\rhatt,\nhatt,\mhatt$. If we maintain and use private versions of these statistics, and we set the privacy level
appropriately, we can ensure JDP.

To do so $\PUCB$ initializes one private counter mechanism for each $\rhatt, \nhatt,\mhatt$ ($2SAH + S^2AH$ counters in total). At episode $t$, we compute the policy using optimism as in $\texttt{UBEV}$, but we use only the private counts $\rtilt,\ntilt,\mtilt$ released from the counter mechanisms. We require that each set of counters is $(\varepsilon/3)$ JDP, and so with $$\Eeps =\counterBound,$$ we can ensure that with probability at least $1-\beta$:
\begin{align*}
\forall t \in [T]: \left|\ntilt(s,a,h) - \nhatt(s,a,h)\right| < \Eeps \enspace,
\end{align*}
where $\nhatt,\ntilt$ are the count and release at the beginning of
the $t$-th episode. The guarantee is uniform in $(s,a,h)$ and also holds simultaneously for $\rtil$ and $\mtil$.

To compute the policy, we define a bonus function $\conftil(s,a,h)$ for each $(s,a,h)$ tuple, which can be decomposed into two parts $\phitilconfL{s,a,h}$ and $\psitilconfL{s,a,h}$, where
\begin{align*}
&\phitilconfL{s,a,h}=\phitilconfR{s,a,h} \enspace,\\
&\psitilconfL{s,a,h}=\psitilconfR{s,a,h} \enspace.
\end{align*}

The term $\phitilconfL{\cdot}$ roughly corresponds to the sampling
error, while $\psitilconfL{\cdot}$ corresponds to errors introduced by
the private counters. Using this bonus function, we use dynamic
programming to compute an optimistic private Q-function in
Algorithm~\ref{alg:privateplanning}. The algorithm here is a standard
batch Q-learning update, with $\conftil(\cdot)$ serving as an optimism
bonus. The resulting Q-function, called $\QtilUB$, encodes a
greedy policy, which we use for the $t$-th episode.

\section{Privacy Analysis of $\PUCB$}

We show that releasing the sequence of actions by
algorithm $\PUCB$ satisfies JDP with respect to any user on an episode
changing his data. Formally,
\begin{theorem}\label{thm:jdp}
Algorithm (\ref{alg:pucb}) \PUCB~is $\eps$-JDP.
\end{theorem}

To prove theorem \ref{thm:jdp}, we use the \emph{billboard lemma} due
to \cite{hsu2016private} which says that an algorithm is JDP if the
output sent to each user is a function of the user's private data and
a common signal computed with standard differential privacy.
We state the formal lemma:
\begin{lemma}[Billboard lemma \cite{hsu2016private}]\label{lem:billboard}
Suppose $\algo:U \rightarrow \cR$ is $\eps$-differentially
private. Consider any set of functions $f_i : U_i\times \cR \rightarrow \cR'$
where $U_i$ is the portion of the database containing the $i$'s user
data. The composition $\{f_i(\Pi_i U, \algo(U))\}$ is
$\eps$-joint differentially private, where $\Pi_i:U\rightarrow U_i$
is the projection to $i$'s data.
\end{lemma}

Let $U_{<t}$ denote the data of all users before episode $t$
and $u_t$ denote the data of the user during episode $t$.
Algorithm  \PUCB keeps track of all events on users $U_{<t}$
in a differentially-private way with private counters
$\rtilt, \ntilt, \mtilt$.
These counters are given to the procedure \PrivQ~ which computes a
$Q$-function
$\QtilUB_t$, and induces the policy $\pit(s,h) \coloneqq \max_a \QtilUB_t(s,a,h)$
to be use by the agent during episode $t$.
Then the output during episode $t$ is generated the policy $\pit$
and the private data of the user $u_t$ according to the protocol
\ref{alg:rlprotocol},
 the output on a single episode is:
 $\left(\pit\left(s^{(t)}_1, 1\right), \ldots,  \pit\left(s^{(t)}_H, H\right)\right)$.
By the billboard lemma \ref{lem:billboard}, the composition of the output
of all T episodes, and the final policy
$\left(\left\{\left( \pit(s^{(t)}_1, 1), \ldots,  \pit(s^{(t)}_H, H)\right) \right\}_{t\in [T]}, \pi_T\right) $
satisfies $\eps$-JDP if the policies
$\{\pit\}_{t\in [T]}$  are computed with a $\eps$-DP mechanism.

Then it only remains to show that the noisy counts satisfy $\eps$-DP.
First, consider the counters for the number of visited states.
 The algorithm $\PUCB$ runs $SAH$
parallel private counters, one for each state tuple $(s,a,h)$.
Each counter is instantiated with a $\eps/(3H)$-differentially
private mechanism which takes an input an event stream
$\nhat(s,a,h) = \{0,1\}^T$ where the $i$th bit is set to 1 if
a user visited the state tuple $(s,a,h)$ during episode $i$ and
0 otherwise. Hence each stream $\nhat(s,a,h)$ is the data
for a private counter. The next claim says that the total $\ell_1$
sensitivity over all streams is bounded by $H$:
\begin{claim}\label{claim:sensitivity}
Let $U, U'$ be two $t$-neighboring user sequences, in the sense that they
are only different in the data for episode $t$.
For each $(s,a,h)\in\states\times\actions\times [H]$,
let $\nhat(s,a,h)$ be the event stream corresponding to user sequence
U and $\nhat'(s,a,h)$  be the event stream corresponding to $U'$.
Then the total $\ell_1$ distance of all stream is given by the following
claim:
\begin{align*}
  \sum_{(s,a,h)\in \states\times\actions\times [H]} \|\nhat(s,a,h)
  -\nhat'(s,a,h)\|_1\leq H
\end{align*}
\end{claim}
\begin{proof}
The proof follows from the fact that on any episode $t$
a user visits at most $H$ states.
\end{proof}

Finally we use a result from \citep[Lemma 34]{hsu2016private}
which states that the
composition of the $SAH$
$(\eps/3H)$-DP counters for $\nhat(\cdot)$ satisfy $(\eps/3)$-DP
as long as the $\ell_1$
sensitivity of the counters is $H$ as shown in
claim \ref{claim:sensitivity}.
We can apply the same analysis to show that the counters corresponding
to the empirical reward $\rhat(\cdot)$ and the transitions $\mhat(\cdot)$
are both also $(\eps/3)$-DP.
Putting it all together releasing the noisy counters is $\eps$-differentially private.

\section{PAC and Regret Analysis of $\PUCB$}

Now that we have established $\PUCB$ is JDP, we turn to utility
guarantees. We establish two forms of utility guarantee namely a PAC
sample complexity bound, and a regret bound. In both cases, comparing
to \texttt{UBEV}, we show that the price for JDP is quite mild. In
both bounds the privacy parameter interacts quite favorably with the
``error parameter.''

We first state the PAC guarantee.
\begin{theorem}[PAC guarantee for $\PUCB$]
\label{thm:PUCBPAC}
Let $\Tmax$ be the maximum number of episodes and $\eps$ the JDP
parameter. Then for any $\alpha \in (0,H]$ and $\beta \in (0,1)$, algorithm $\PUCB$ with parameters $(\eps,\beta)$ follows a policy that with
probability at least $1-\beta$ is $\alpha$-optimal on all but
\begin{align*}
O\left(\left(\frac{SAH^4}{\alpha^2} + \frac{S^2AH^4}{\eps\alpha}\right)
\polylog\left(\Tmax,S,A,H,\tfrac{1}{\alpha},\tfrac{1}{\beta}, \tfrac{1}{\eps}\right)\right)
\end{align*}
episodes.
\end{theorem}
The theorem states that if we run $\PUCB$ for many episodes, it will act
near-optimally in a large fraction of them. The number of episodes
where the algorithm acts suboptimally scales polynomially with all the
relevant parameters. In particular, notice that in terms of the
utility parameter $\alpha$, the bound scales as
$1/\alpha^2$. In fact the first term here matches the
guarantee for the non-private algorithm \texttt{UBEV} up to
polylogarithmic factors. On the other hand, the privacy parameter
$\eps$ appears only in the term scaling as $1/\alpha$. In the common
case where $\alpha$ is relatively small, this term is typically of a
lower order, and so the price for privacy here is relatively low.

Analogous to the PAC bound, we also have a regret guarantee.
\begin{theorem}[Regret bound for \PUCB]
\label{thm:PUCBRegret}
With probability at least $1-\beta$, the regret of $\PUCB$ up to episode $T$
 is at most
\begin{align*}
  O\left(\left(H^2\sqrt{SAT} + \frac{SAH^3 + S^2AH^3}{\eps}\right)
   \polylog\left(T, S, A, H, \tfrac{1}{\beta}, \tfrac{1}{\eps}\right)\right) \enspace.
\end{align*}
\end{theorem}

A similar remark to the PAC bound applies here: the privacy parameter only
appears in the $\polylog(T)$ terms, while the leading order term
scales as $\sqrt{T}$. In this guarantee it is clear that as $T$ gets
large, the utility price for privacy is essentially negligible.

We also remark that both bounds have ``lower order'' terms that scale
with $S^2$. This is quite common for tabular reinforcement
algorithms~\cite{dann2017unifying,azar2017minimax}. We find it quite
interesting to observe that the privacy parameter $\eps$ interacts
with this term, but not with the so-called ``leading'' term in these
guarantees.

\paragraph{Proof Sketch.}
The proofs for both results parallel the arguments
in~\citep{dann2017unifying} for the analysis of \texttt{UBEV}. The
main differences arises from the fact that we have adjusted the
confidence interval $\conftil$ to account for the noise in the
releases of
$\rtil,\ntil,\mtil$. In~\citep{dann2017unifying} the bonus
is crucially used to establish optimism, and the final guarantees are
related to the over-estimation incurred by these bonuses.  We focus on
these two steps in this sketch, with a full proof deferred to the
appendix.

First we verify optimism. Fix episode $t$ and state tuple $(s,a,h)$, and let us
abbreviate the latter simply by $\sah$. Assume that $\Vtil_{h+1}$ is
private and optimistic in the sense that $\Vtil_{h+1}(s) \geq V_{h+1}^*(s)$,
for all $s \in \states$. First define the empirical Q-value
\begin{align*}
	\Qhat_t(\sah) =
	\frac{\rhatt(\sah)
		+ \sum_{s'\in\states}\Vtil_{h+1}(s')\mhatt(\sah,s')}{\nhatt(\sah)} \enspace.
\end{align*}
The optimistic Q-function, which is similar to the one used by~\citep{dann2017unifying}, is given
by
\begin{align*}
	\QhatUB_t(\sah) = \Qhat_t(\sah)
		+ (H+1)\phihatconfL{\sah} \enspace,
\end{align*}
where $\phihatconfL{\sah} = \phihatconfR{\sah}$. A standard
concentration argument shows that $\QhatUB_t\geq Q^\star$, assuming that
$\Vtil_{h+1} \geq V^\star_{h+1}$.

Of course, both $\Qhat_t$ and $\QhatUB_t$ involve the non-private counters
$\rhat,\nhat,\mhat$, so they are \emph{not} available to our
algorithm. Instead, we construct a surrogate for the empirical
Q-value using the private releases:
\begin{align*}
&\Qtil_t(\sah)=
\frac{\rtilt(\sah)
+\sum_{s'\in\states}\Vtil_{h+1}(s')\mtilt(\sah,s')}{\ntilt(\sah)} \enspace.
\end{align*}

Our analysis involves relating $\Qtil_t$ which the algorithm has access
to, with $\Qhat_t$ which is non-private. To do this, note that by the
guarantee for the counting mechanism, we have
\begin{align}
\Qhat_t(\sah)
\leq\frac{\rtilt(\sah)+\Eeps
+\sum_{s'\in\states}\Vtil_{h+1}(s')(\mtilt(\sah,s')+\Eeps)}{\ntilt(\sah)-\Eeps} \enspace.\label{eq:qhat_upper}
\end{align}
Next, we use the following elementary fact.

\begin{claim}
\label{claim:main:oneoverNbound}
Let $y \in \mathbb{R}$ be any positive real number.
Then for all $x \in \mathbbm{R}$ with $x\geq 2y$ it holds that
$\frac{1}{x - y} \leq \frac{1}{x} + \frac{2y}{x^2}$.
\end{claim}

If $\ntilt(\sah) \geq 2\Eeps$, then we can
apply claim~\ref{claim:main:oneoverNbound}
to equation~\eqref{eq:qhat_upper}, along with the facts that $\Vtil_{h+1}(s')
\leq H$ and $\rtilt(\sah) \leq \ntilt(\sah) + 2\Eeps\leq 2\ntilt(\sah)$, to upper bound $\Qhat_t$ by
$\Qtil_t$. This gives:
\begin{align*}
\Qhat_t(\sah) &\leq \Qtil_t(\sah) + \left(\frac{1}{\ntilt(\sah)} +
\frac{2\Eeps}{\ntilt(\sah)^2}\right)\cdot(1+SH)\Eeps\\
 & = \Qtil_t(\sah) + \psitilconfL{\sah} \enspace.
\end{align*}
Therefore, we see that $\Qtil_t(\sah)+\psitilconfL{\sah}$ dominates
$\Qhat_t(\sah)$. Accordingly, if we inflate by
$\phitilconfL{\sah}$ -- which is clearly an upper bound on
$\phihatconfL{\sah}$ -- we account for the statistical fluctuations and
can verify optimism. In the event that $\ntil_t(\sah) \leq
2\Eeps$, we simply upper bound $Q^* \leq H$.

For the over-estimation, the bonus we have added is
$\phitilconfL{\sah} + \psitilconfL{\sah}$, which is closely related to
the original bonus $\phihatconfL{\sah}$. The essential property for
our bonus is that it is not significantly larger than the original one
$\phihatconfL{\sah}$. Indeed, $\phihatconfL{\sah}$ scales as
$1/\sqrt{\ntil_t(\sah)}$ while $\psitilconfL{\sah}$ scales roughly
as $\Eeps/\ntil_t(\sah) + \Eeps^2/\ntil_t(\sah)^2$, which is lower
order in the dependence on $\ntil_t(\sah)$. Similarly, the other
sources of error here only have lower order effects on the
over-estimation.


In detail, there are three sources of error. First,
$\phitilconfL{\sah}$ is within a constant factor of
$\phihatconfL{\sah}$ since we are focusing on rounds where
$\ntil_t(\sah) \geq 2\Eeps$. Second, as the policy suboptimality
is related to the bonuses on the states and actions we are likely to
visit, we cannot have many rounds where $\ntil_t(\sah) \leq
2\Eeps$, since all of the private counters are increasing. A similar
argument applies for $\psitilconfL{\sah}$: we can ignore states that
we visit infrequently, and the private counters $\ntil_t(\sah)$
for states that we visit frequently increase rapidly enough to
introduce minimal additional error. Importantly, in the latter two
arguments, we have terms of the form $\Eeps/\ntil_t(\sah)$, while
$\phihatconfL{\sah}$ itself scales as $\sqrt{1/\nhat_t(\sah)}$,
which dominates in terms of the accuracy parameter $\alpha$ or the
number of episodes $T$. As such we obtain PAC and regret guarantees
where the privacy parameter $\eps$ does not appear in the dominant
terms.

\section{Lower Bounds}\label{sec:lower}

In this section we prove the following lower bounds on the sample
complexity and regret for any PAC RL agent providing joint
differential privacy.

\begin{theorem}[PAC Lower Bound]\label{thm:lowerbound}
Let $\algo$ be an RL agent satisfying $\eps$-JDP.
Suppose that $\algo$ is $(\alpha, \beta)$-PAC for some $\beta \in (0,1/8)$.
Then, there exists a fixed-horizon episodic MDP where the number of episodes until the algorithm's policy is $\alpha$-optimal with probability at least $1 - \beta$ satisfies
\begin{align*}
\Ex{n_\algo} \geq \Omega\left( \frac{SAH^2}{\alpha^2} + \frac{SAH}{\alpha\epsilon}
\ln\left(\frac{1}{\beta}\right)\right) \enspace.
\end{align*}
\end{theorem}

\begin{theorem}[Private Regret Lower Bound]\label{thm:regretlower}
  For any $\eps$ JDP-algorithm $\algo$ there exist an MDP $M$ with $S$
  states $A$ actions over $H$ time steps per episode such that
  for any initial state $s\in\states$ the
  expected regret of $\algo$ after $T$ steps is
\begin{align*}
\Ex{\regret(T)} =
{\Omega}\left (\sqrt{HSA T} + \frac{S A H\log(T)}{\eps}\right)
\end{align*}
for any $T \geq S^{1.1}$.
\end{theorem}


Here we present the proof steps for the sample complexity lower bound in Theorem~\ref{thm:lowerbound}. The proof for the regret lower bound in Theorem~\ref{thm:regretlower} follows from a similar argument and is deferred to the appendix.

To obtain Theorem~\ref{thm:lowerbound}, we go through two intermediate lower bounds: one for private best-arm identification in multi-armed bandits problems (Lemma~\ref{lem:privMAB}), and one for private RL in a relaxed scenario where the initial state of each episode is considered public information (Lemma~\ref{lem:lowerboundpublic}).
At first glance our arguments look similar to other techniques that provide lower bounds for RL in the non-private setting by leveraging lower bounds for bandits problems, e.g.\ \cite{strehl2009reinforcement,dann2015sample}.
However, getting this strategy to work in the private case is significantly more challenging because one needs to ensure the notions of privacy used in each of the lower bounds are compatible with each other.
Since this is the main challenge to prove Theorem~\ref{thm:lowerbound}, we focus our presentation on the aspects that make the private lower bound argument different from the non-private one, and defer the rest of details to the appendix.

\subsection{Lower Bound for Best-Arm Identification}\label{sec:mablb}

The first step is a lower bound for best-arm identification for differentially private multi-armed bandits algorithms.
This considers mechanisms $\algo$ interacting with users via the MAB protocol described in \cref{alg:mabprotocol}, where we assume arms $a^{(t)}$ come from some finite space $\actions$ and rewards are binary, $r^{(t)} \in \{0,1\}$.
Recall that $\Tmax$ denotes the total number of users.
Our lower bound applies to mechanisms for this protocol that satisfy standard DP in the sense that the adversary has access to all the outputs $\algo(U) = (a^{(1)}, \ldots, a^{(\Tmax)},\hat{a})$ produced by the mechanism.

\begin{definition}
A MAB mechanism $\algo$ is $\epsilon$-DP if for any neighboring user sequences $U$ and $U'$ differing in a single user, and all events $E \subseteq \actions^{\Tmax+1}$ we have
\begin{align*}
\Pr[\algo(U) \in E] \leq e^{\epsilon} \Pr[\algo(U') \in E]  \enspace.
\end{align*}
\end{definition}

To measure the utility of a mechanism for performing \emph{best-arm identification} in MABs we consider a stochastic setting with independent arms. In this setting each arm $a \in \actions$ produces rewards following a Bernoulli distribution with expectation $\bar{P}_a$ and the goal is to identify high probability an optimal arm $a^*$ with expected reward $\bar{P}_{a^*} = \max_{a \in \actions} \bar{P}_a$.
A problem instance can be identified with the vector of expected rewards $\bar{P} = (\bar{P}_a)_{a \in \actions}$.

\begin{algorithm}[h]
\caption{MAB Protocol for Best-Arm Identification}\label{alg:mabprotocol}
\KwIn{Agent $\algo$ and users $u_1, \ldots, u_\Tmax$}
\For{$t \in [\Tmax]$}{
  $\algo$ sends arm $a^{(t)}$ to $u_t$ \\
  $u_t$ sends reward $r^{(t)}$ to $\algo$
}
$\algo$ releases arm $\hat{a}$
\end{algorithm}

The lower bound result relies on the following adaptation of the coupling lemma from \citep[Lemma 6.2]{karwa2017finite}.
\begin{lemma}\label{lem:KarwaVadhanMAB}
Fix any arm $a \in [k]$. Now consider any pair of MAB instances $\mu, \nu \in [0,1]^k$ both with $k$ arms and time horizon $T$, such that $\|\mu_a - \nu_a \|_{tv} < \alpha$ and  $\|\mu_{a'} - \nu_{a'} \|_{tv} = 0$ for all $a' \neq a$. Let $R \sim \bernoulli(\mu)^T$ and $Q \sim \bernoulli(\nu)^T$ be the sequence of $T$ rounds of rewards sampled under $\mu$ and $\nu$ respectively, and let $\algo$ be any $\epsilon$-DP multi-armed bandit algorithm. Then, for any event $E$ such that under event $E$ arm $a$ is pulled less than $t$ times,
\begin{align*}
\prob{\algo,R}{E} \leq e^{6 \epsilon t \alpha}\prob{\algo,Q}{E}
\end{align*}
\end{lemma}
%
%
\begin{lemma}[Private MAB Lower Bound]\label{lem:privMAB}
Let $\algo$ be a MAB best-arm identification algorithm satisfying $\epsilon$-DP that succeeds with probability at least $1-\beta$, for some $\beta \in (0,1/4)$.
For any MAB instance $\bar{P}$ and any $\alpha$-suboptimal arm $a$ with $\alpha > 0$ (i.e.\ $\bar{P}_a = \bar{P}_{a^*} - \alpha$), the number of times that $\algo$ pulls arm $a$ during the protocol satisfies
\begin{align*}
\Ex{n_a} > \frac{1}{24 \epsilon \alpha}\ln{\left(\frac{1}{4\beta}\right)} \enspace.
\end{align*}
\end{lemma}
\begin{proof}
  Let $a^*$ be the optimal arm under $\bar{P}$ and $a$ an $\alpha$-suboptimal arm.
We construct an alternative MAB instance $\bar{Q}$ by exchanging the rewards of $a$ and $a^*$:
$\bar{Q}_a = \bar{P}_{a^*}$, $\bar{Q}_{a^*} = \bar{P}_a$, and the rest of rewards are identical on both instances. Note that now $a^*$ is $\alpha$-suboptimal under $\bar{Q}$.


Let $t_a = \frac{1}{24\epsilon \alpha}\ln{\left(\frac{1-2\beta}{2\beta}\right)}$ and $n_a$ is the number of times the policy $\algo$ pulls arm $a$.
We suppose that $\ex{\bar{P}}{n_a} \leq t_a$ and derive a contradiction.

Define $A$ to be the event that arm $n_a$ is pulled less than $4 t_a$ times, that is $A := \{n_a \leq 4 t_a\}$. From Markov's inequality we have
\begin{align}
\label{eq:assump}
t_a \geq \ex{\bar{P}}{n_a} &\geq 4 t_a \prob{\bar{P}}{n_a > 4 t_a} \\
\label{eq:A}
& = 4 t_a\left(1- \prob{\bar{P}}{n_a \leq 4 t_a} \right) \enspace,
\end{align}
where the first inequality \eqref{eq:assump} comes from the assumption that $\Ex{n_a} \leq t_a$.
From \eqref{eq:A} above it follows that $\prob{\bar{P}}{A} \geq 3/4$.
We also let $B$ be the event that arm $a^*$ is selected.
Since arm $a^*$ is optimal under $\bar{P}$, our assumption on $\algo$ implies $\prob{\bar{P}}{B} \geq 1-\beta$.

Now let $E$ be the event that both $A$ and $B$ occur, that is $E = A\cap B$.
We combine the lower bound of $\prob{\bar{P}}{A}$ and $\prob{\bar{P}}{B}$ to get a lower bound for $\prob{\bar{P}}{E}$. First we show that $\prob{\bar{P}}{B|A} \geq 3/4 -\beta$:
\begin{align*}
1-\beta
&\leq
\prob{\bar{P}}{B|A}\prob{\bar{P}}{A} + \prob{\bar{P}}{B|A^c}\prob{\bar{P}}{A^c} \\
&\leq
\prob{\bar{P}}{B|A} + \prob{\bar{P}}{A^c}
\leq \prob{\bar{P}}{B|A} + 1/4 \enspace. \\
\end{align*}

By replacing in the lower bounds for $\prob{\bar{P}}{A}$ and $\prob{\bar{P}}{B|A}$ we obtain:
\begin{align*}
\prob{\bar{P}}{E} &= \prob{\bar{P}}{A}\prob{\bar{P}}{B|A} \geq \frac{3}{4} \left(\frac{3}{4} - \beta \right) \enspace.
\end{align*}

On instance $\bar{Q}$ arm $a^*$ is suboptimal, hence
we have that $\prob{\bar Q}{E} \leq \beta$.
Now we apply the group privacy property (Lemma~\ref{lem:KarwaVadhanMAB}) where the number of observations is $4 t_a$
and $t_a = \frac{1}{24\eps\alpha}\ln\left(\frac{1/2 -\beta}{\beta}\right)$
to obtain
\begin{align}
\frac{3}{4} \left(\frac{3}{4} - \beta \right)
&\leq
\prob{\bar{P}}{E}
\leq e^{6\epsilon \alpha 4 t_a} \prob{\bar{Q}}{E} \notag \\
&\leq
e^{6\epsilon \alpha 4 t_a} \beta
=
\frac{1}{2}-\beta \label{eq:grouplem}
\enspace.
\end{align}

But $\frac{3}{4} \left(\frac{3}{4} - \beta \right)>\frac{1}{2}-\beta$
for $\beta \in (0,1/4)$, therefore \eqref{eq:grouplem}
is a contradiction.
\end{proof}
%
%
%
\subsection{Lower Bound for RL with Public Initial State}\label{sec:rlpslb}

To leverage the lower bound for private best-arm identification in the RL setting we first consider a simpler setting where the initial state of each episode is public information. This means that we consider agents $\algo$ interacting with a variant of the protocol in Algorithm~\ref{alg:rlprotocol} where each user $t$ releases their first state $s_1^{(t)}$ in addition to sending it to the agent.
We model this scenario by considering agents whose inputs $(U,S_1)$ include the sequence of initial states $S_1 = (s_1^{(1)},\ldots,s_1^{(\Tmax)})$, and define the privacy requirements in terms of a different notion of neighboring inputs: two sequences of inputs $(U,S_1)$ and $(U',S_1')$ are $t$-neighboring if $u_{t'} = u'_{t'}$ for all $t \neq t'$ and $S_1 = S_1'$.
That is, we do not expect to provide privacy in the case where the user that changes between $U$ and $U'$ also changes their initial state, since in this case making the initial state public already provides evidence that the user changed.
Note, however, that $u_t$ and $u_t'$ can provide different rewards for actions taken by the agent on state $s_1^{(t)}$.

\begin{definition}\label{def:jdppublic}
A randomized RL agent $\algo$ is $\epsilon$-JDP under continual observation in the \emph{public initial state} setting if for all $t \in [\Tmax]$, all $t$-neighboring user-state sequences $(U,S_1)$, $(U',S_1')$, and all events $E \subseteq A^{H \times [\Tmax-1]} \times \Pi$ we have
\begin{align*}
\pr{\algo_{-t}(U,S_1) \in E }
\leq e^\epsilon \pr{\algo_{-t}(U',S_1') \in E} \enspace.
\end{align*}
\end{definition}

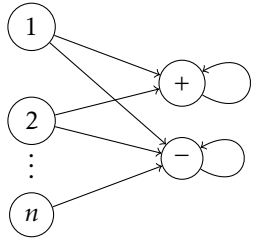
\begin{figure}[h]
\begin{equation*}
 \begin{tikzpicture}
	\begin{pgfonlayer}{nodelayer}
		\node [style=States] (1) at (0, 2.5) {$1$};
		\node [style=States] (2) at (0, 0) {$2$};
		\node [style=States] (3) at (0, -2.5) {$n$};
		\node [style=States] (4) at (4, 1) {$+$};
		\node [style=States] (5) at (4, -1) {$-$};
		\node [style=none] (7) at (0, -1) {$\vdots$};
	\end{pgfonlayer}
	\begin{pgfonlayer}{edgelayer}
		\draw [style=arrow] (1) to (4);
		\draw [style=arrow] (1) to (5);
		\draw [style=arrow] (2) to (4);
		\draw [style=arrow] (2) to (5);
		\draw [style=arrow] (3) to (5);
		\draw [style=arrow, in=30, out=-30, loop] (4) to ();
		\draw [style=arrow, in=30, out=-30, loop] (5) to ();
	\end{pgfonlayer}
\end{tikzpicture}
\end{equation*}
\caption{Class of hard MDP instances used in the lower bound.}\label{fig:hardMDP}
\end{figure}

We obtain a lower bound on the sample complexity of PAC RL agents that
satisfy JDP in the public initial state setting by constructing a
class of hard MDPs shown in Figure~\ref{fig:hardMDP}.  An MDP in this
class has state space $\states \coloneqq [n] \cup \{+,-\}$ and action
space $\actions \coloneqq \{0,\ldots,m\}$.  On each episode, the agent
starts on one of the initial states $\{1, \ldots, n\}$ chosen
uniformly at random.  On each of the initial states the agent has
$m+1$ possible actions and transitions can only take it to one of two
possible absorbing states $\{+,-\}$.  Lastly, if the current state is
either one of $\{ +, - \}$ then the only possible transition is a self
loop, hence the agent will in that state until the end of the
episode. We assume in these absorbing states the agent can only take a
fixed action.  Every action which transitions to state $+$ provides
reward $1$ while actions transitioning to state $-$ provide reward
$0$.  In particular, in each episode the agent either receives reward
$H$ or $0$.

Such an MDP can be seen as consisting of $n$ parallel MAB problems.
Each MAB problem determines the transition probabilities between the
initial state $s \in\{1, \ldots, n\}$ and the absorbing states
$\{+,-\}$.  We index the possible MAB problems in each initial state
by their optimal arm, which is always one of $\{0,\ldots,m\}$.  We
write $I_s \in \{0,\ldots,m\}$ to denote the MAB instance in initial
state $s$, and define the transition probabilities such that
$\pr{+|s,0} = 1/2+\alpha'/2$ and $\pr{+|s,a'} = 1/2$ for $a' \neq I_s$
for all $I_s$, and for $I_s \neq 0$ we also have
$\pr{+|s,I_s} = 1/2 + \alpha'$.  Here $\alpha'$ is a free parameter to
be determined later.  We succinctly represent an MDP in the class by
identifying the optimal action (i.e.\ arm) in each initial state:
$I \coloneqq (I_1,\ldots,I_n)$.

To show that our MAB lower bounds imply lower bounds for an RL agent interacting with MDPs in this class we prove that collecting the first action taken by the agent in all episodes $t$ with a fixed initial state $s_1^{(t)} = s \in [n]$ simulates the execution of an $\epsilon$-DP MAB algorithm.


Let $\algo$ be an RL agent and $(U,S_1)$ a user-state input sequence with initial states from some set $\states_1$. Let $\algo(U,S_1) = (\vec{a}^{(1)},\ldots,\vec{a}^{(\Tmax)},\pi) \in \actions^{H \times \Tmax} \times \Pi$ be the collection of all outputs produced by the agent on inputs $U$ and $S_1$.
For every $s \in \states_1$ we write $\algo_{1,s}(U,S_1)$ to denote the restriction of the previous trace to contain just the first action from all episodes starting with $s$ together with the action predicted by the policy at states $s$:
\begin{align*}
\algo_{1,s}(U,S_1) \coloneqq \left(a_1^{(t_{s,1})}, \ldots, a_1^{(t_{s,T_s})}, \pi(s)\right) \enspace,
\end{align*}
where $T_s$ is the number of occurrences of $s$ in $S_1$ and $t_{s,1}, \ldots, t_{s,T_s}$ are the indices of these occurrences.
Furthermore, given $s \in \states_1$ we write $U_s = (u_{t_{s,1}},\ldots,u_{t_{s,T_s}})$ to denote the set of users whose initial state equals $s$.

\begin{lemma}\label{lem:jdptodp}
Let $(U,S_1)$ be a user-state input sequence with initial states from some set $\states_1$.
Suppose $\cM$ is an RL agent that satisfies $\eps$-JDP in the public initial state setting.
Then, for any $s \in \states_1$ the trace $\algo_{1,s}(U,S_1)$ is the output of an $\epsilon$-DP MAB mechanism on input $U_s$.
\end{lemma}

Using Lemmas~\ref{lem:privMAB} and~\ref{lem:jdptodp} and a reduction from RL lower bounds to bandits lower bounds yields the second term in the following result.
The first terms follows directly from the non-private lower bound in \cite{dann2015sample}.

\begin{lemma}\label{lem:lowerboundpublic}
Let $\algo$ be an RL agent satisfying $\eps$-JDP in the public initial state setting.
Suppose that $\algo$ is $(\alpha, \beta)$-PAC for some $\beta \in (0,1/8)$.
Then, there exists a fixed-horizon episodic MDP where the number of episodes until the algorithm's policy is $\alpha$-optimal with probability at least $1 - \beta$ satisfies
\begin{equation*}
\Ex{n_\algo} \geq \Omega\left( \frac{S A H^2}{\alpha^2} + \frac{S A H}{\alpha\epsilon}\ln\left(\frac{1}{\beta} \right)\right) \enspace.
\end{equation*}
\end{lemma}

Finally, Theorem~\ref{thm:lowerbound} follows from Lemma~\ref{lem:lowerboundpublic} by observing that any RL agent $\algo$ satisfying $\eps$-JDP also satisfies $\eps$-JDP in the public state setting (see lemma \ref{lem:publicvsprivatestates} and see appendix for proof).

\begin{lemma}\label{lem:publicvsprivatestates}
Any RL agent $\algo$ satisfying $\eps$-JDP also satisfies $\eps$-JDP in the public state setting.
\end{lemma}
%


\section{Conclusion}
In this paper, we initiate the study of differentially private
algorithms for reinforcement learning. On the conceptual level, we
formalize the privacy desiderata via the notion of joint differential
privacy, where the algorithm cannot strongly base future decisions off
sensitive information from previous interactions. Under this
formalism, we provide a JDP algorithm and establish both PAC and
regret utility guarantees for episodic tabular MDPs. Our
results show that the utility cost for privacy is asymptotically
negligible in the large accuracy regime. We also establish the first lower bounds for reinforcement learning with JDP.

A natural direction for future work is to close the gap between our
upper and lower bounds. A similar gap remains open for tabular RL
\emph{without} privacy considerations, but the setting is more
difficult with privacy, so it may be easier to establish a lower bound
here. We look forward to pursuing this direction, and hope that
progress will yield new insights into the non-private setting.

Beyond the tabular setup considered in this paper, we believe that designing RL algorithms providing state and reward privacy in non-tabular settings is a promising direction for future work with considerable potential for real-world applications.

\section{Acknowledgements}
Giuseppe Vietri has been supported by the GAANN fellowship from the U.S. Department of Education.
We want to thank Matthew Joseph, whose comments improved our definition of joint-differential-privacy.
\newpage

\bibliographystyle{alpha}
\bibliography{main}

\newcommand{\etalchar}[1]{$^{#1}$}
\begin{thebibliography}{KMA{\etalchar{+}}19}

\bibitem[Abo18]{abowd2018us}
John~M Abowd.
\newblock The us census bureau adopts differential privacy.
\newblock In {\em Proceedings of the 24th ACM SIGKDD International Conference
  on Knowledge Discovery \& Data Mining}, pages 2867--2867, 2018.

\bibitem[ACP19]{tfprivacy}
Galen Andrew, Steve Chien, and Nicolas Papernot.
\newblock Tensorflow privacy.
\newblock \url{https://github.com/tensorflow/privacy}, 2019.

\bibitem[AOM17]{azar2017minimax}
Mohammad~Gheshlaghi Azar, Ian Osband, and R{\'e}mi Munos.
\newblock Minimax regret bounds for reinforcement learning.
\newblock In {\em Proceedings of the 34th International Conference on Machine
  Learning-Volume 70}, pages 263--272. JMLR. org, 2017.

\bibitem[AS17]{agarwal2017price}
Naman Agarwal and Karan Singh.
\newblock The price of differential privacy for online learning.
\newblock In {\em Proceedings of the 34th International Conference on Machine
  Learning-Volume 70}, pages 32--40. JMLR. org, 2017.

\bibitem[BDT19]{basu2019differential}
Debabrota Basu, Christos Dimitrakakis, and Aristide Tossou.
\newblock Differential privacy for multi-armed bandits: What is it and what is
  its cost?
\newblock {\em arXiv preprint arXiv:1905.12298}, 2019.

\bibitem[BGP16]{DBLP:conf/icml/BalleGP16}
Borja Balle, Maziar Gomrokchi, and Doina Precup.
\newblock Differentially private policy evaluation.
\newblock In {\em Proceedings of the 33nd International Conference on Machine
  Learning, {ICML} 2016, New York City, NY, USA, June 19-24, 2016}, pages
  2130--2138, 2016.

\bibitem[CSS11]{chan2011private}
T-H~Hubert Chan, Elaine Shi, and Dawn Song.
\newblock Private and continual release of statistics.
\newblock {\em ACM Transactions on Information and System Security (TISSEC)},
  14(3):26, 2011.

\bibitem[DB15]{dann2015sample}
Christoph Dann and Emma Brunskill.
\newblock Sample complexity of episodic fixed-horizon reinforcement learning.
\newblock In {\em Advances in Neural Information Processing Systems}, pages
  2818--2826, 2015.

\bibitem[DLB17]{dann2017unifying}
Christoph Dann, Tor Lattimore, and Emma Brunskill.
\newblock Unifying pac and regret: Uniform pac bounds for episodic
  reinforcement learning.
\newblock In {\em Advances in Neural Information Processing Systems}, pages
  5713--5723, 2017.

\bibitem[DMNS06]{dwork2006calibrating}
Cynthia Dwork, Frank McSherry, Kobbi Nissim, and Adam Smith.
\newblock Calibrating noise to sensitivity in private data analysis.
\newblock In {\em Theory of cryptography conference}, pages 265--284. Springer,
  2006.

\bibitem[DNPR10]{dwork2010differential}
Cynthia Dwork, Moni Naor, Toniann Pitassi, and Guy~N Rothblum.
\newblock Differential privacy under continual observation.
\newblock In {\em Proceedings of the forty-second ACM symposium on Theory of
  computing}, pages 715--724. ACM, 2010.

\bibitem[EPK14]{erlingsson2014rappor}
{\'U}lfar Erlingsson, Vasyl Pihur, and Aleksandra Korolova.
\newblock Rappor: Randomized aggregatable privacy-preserving ordinal response.
\newblock In {\em Proceedings of the 2014 ACM SIGSAC conference on computer and
  communications security}, pages 1054--1067, 2014.

\bibitem[HBAL19]{DBLP:journals/corr/abs-1907-02444}
Naoise Holohan, Stefano Braghin, P{\'{o}}l~Mac Aonghusa, and Killian Levacher.
\newblock Diffprivlib: The {IBM} differential privacy library.
\newblock {\em CoRR}, abs/1907.02444, 2019.

\bibitem[HHR{\etalchar{+}}16]{hsu2016private}
Justin Hsu, Zhiyi Huang, Aaron Roth, Tim Roughgarden, and Zhiwei~Steven Wu.
\newblock Private matchings and allocations.
\newblock {\em SIAM Journal on Computing}, 45(6):1953--1984, 2016.

\bibitem[JOA10]{jaksch2010near}
Thomas Jaksch, Ronald Ortner, and Peter Auer.
\newblock Near-optimal regret bounds for reinforcement learning.
\newblock {\em Journal of Machine Learning Research}, 11(Apr):1563--1600, 2010.

\bibitem[Kak03]{Kakade2003}
Kakade.
\newblock On the sample complexity of reinforcement learning.
\newblock {\em Diss. University of London}, 2003.

\bibitem[KMA{\etalchar{+}}19]{kairouz2019advances}
Peter Kairouz, H.~Brendan McMahan, Brendan Avent, Aurélien Bellet, Mehdi
  Bennis, Arjun~Nitin Bhagoji, Keith Bonawitz, Zachary Charles, Graham Cormode,
  Rachel Cummings, Rafael G.~L. D'Oliveira, Salim~El Rouayheb, David Evans,
  Josh Gardner, Zachary Garrett, Adrià Gascón, Badih Ghazi, Phillip~B.
  Gibbons, Marco Gruteser, Zaid Harchaoui, Chaoyang He, Lie He, Zhouyuan Huo,
  Ben Hutchinson, Justin Hsu, Martin Jaggi, Tara Javidi, Gauri Joshi, Mikhail
  Khodak, Jakub Konečný, Aleksandra Korolova, Farinaz Koushanfar, Sanmi
  Koyejo, Tancrède Lepoint, Yang Liu, Prateek Mittal, Mehryar Mohri, Richard
  Nock, Ayfer Özgür, Rasmus Pagh, Mariana Raykova, Hang Qi, Daniel Ramage,
  Ramesh Raskar, Dawn Song, Weikang Song, Sebastian~U. Stich, Ziteng Sun,
  Ananda~Theertha Suresh, Florian Tramèr, Praneeth Vepakomma, Jianyu Wang,
  Li~Xiong, Zheng Xu, Qiang Yang, Felix~X. Yu, Han Yu, and Sen Zhao.
\newblock Advances and open problems in federated learning, 2019.

\bibitem[KPRU14]{DBLP:conf/innovations/KearnsPRU14}
Michael~J. Kearns, Mallesh~M. Pai, Aaron Roth, and Jonathan Ullman.
\newblock Mechanism design in large games: incentives and privacy.
\newblock In {\em Innovations in Theoretical Computer Science, ITCS'14,
  Princeton, NJ, USA, January 12-14, 2014}, pages 403--410, 2014.

\bibitem[KV17]{karwa2017finite}
Vishesh Karwa and Salil Vadhan.
\newblock Finite sample differentially private confidence intervals.
\newblock {\em arXiv preprint arXiv:1711.03908}, 2017.

\bibitem[LSTS19]{li2019federated}
Tian Li, Anit~Kumar Sahu, Ameet Talwalkar, and Virginia Smith.
\newblock Federated learning: Challenges, methods, and future directions, 2019.

\bibitem[MT15]{mishra2015nearly}
Nikita Mishra and Abhradeep Thakurta.
\newblock (nearly) optimal differentially private stochastic multi-arm bandits.
\newblock In {\em Proceedings of the Thirty-First Conference on Uncertainty in
  Artificial Intelligence}, pages 592--601, 2015.

\bibitem[NR18]{DBLP:conf/icml/NeelR18}
Seth Neel and Aaron Roth.
\newblock Mitigating bias in adaptive data gathering via differential privacy.
\newblock In {\em Proceedings of the 35th International Conference on Machine
  Learning, {ICML} 2018, Stockholmsm{\"{a}}ssan, Stockholm, Sweden, July 10-15,
  2018}, pages 3717--3726, 2018.

\bibitem[PWZ{\etalchar{+}}19]{DBLP:conf/atal/PanWZLYS19}
Xinlei Pan, Weiyao Wang, Xiaoshuai Zhang, Bo~Li, Jinfeng Yi, and Dawn Song.
\newblock How you act tells a lot: Privacy-leaking attack on deep reinforcement
  learning.
\newblock In {\em Proceedings of the 18th International Conference on
  Autonomous Agents and MultiAgent Systems, {AAMAS} '19, Montreal, QC, Canada,
  May 13-17, 2019}, pages 368--376, 2019.

\bibitem[SB18]{sutton2018reinforcement}
Richard~S Sutton and Andrew~G Barto.
\newblock {\em Reinforcement learning: An introduction}.
\newblock MIT press, 2018.

\bibitem[SLL09]{strehl2009reinforcement}
Alexander~L Strehl, Lihong Li, and Michael~L Littman.
\newblock Reinforcement learning in finite mdps: Pac analysis.
\newblock {\em Journal of Machine Learning Research}, 10(Nov):2413--2444, 2009.

\bibitem[SS18]{shariff2018differentially}
Roshan Shariff and Or~Sheffet.
\newblock Differentially private contextual linear bandits.
\newblock In {\em Advances in Neural Information Processing Systems}, pages
  4296--4306, 2018.

\bibitem[TD16]{tossou2016algorithms}
Aristide~CY Tossou and Christos Dimitrakakis.
\newblock Algorithms for differentially private multi-armed bandits.
\newblock In {\em Thirtieth AAAI Conference on Artificial Intelligence}, 2016.

\bibitem[TD17]{tossou2017achieving}
Aristide Charles~Yedia Tossou and Christos Dimitrakakis.
\newblock Achieving privacy in the adversarial multi-armed bandit.
\newblock In {\em Thirty-First AAAI Conference on Artificial Intelligence},
  2017.

\bibitem[TD18]{tossou2018differential}
Aristide~CY Tossou and Christos Dimitrakakis.
\newblock On the differential privacy of thompson sampling with gaussian prior.
\newblock {\em arXiv preprint arXiv:1806.09192}, 2018.

\bibitem[Tea17]{appledp}
Apple Differential~Privacy Team.
\newblock Learning with privacy at scale.
\newblock
  \url{https://machinelearning.apple.com/2017/12/06/learning-with-privacy-at-scale.html},
  2017.

\bibitem[TS13]{thakurta2013nearly}
Abhradeep~Guha Thakurta and Adam Smith.
\newblock (nearly) optimal algorithms for private online learning in
  full-information and bandit settings.
\newblock In {\em Advances in Neural Information Processing Systems}, pages
  2733--2741, 2013.

\bibitem[WH19]{NIPS2019_9310}
Baoxiang Wang and Nidhi Hegde.
\newblock Privacy-preserving q-learning with functional noise in continuous
  spaces.
\newblock In H.~Wallach, H.~Larochelle, A.~Beygelzimer, F.~d'Alch\'{e} Buc,
  E.~Fox, and R.~Garnett, editors, {\em Advances in Neural Information
  Processing Systems 32}, pages 11323--11333. 2019.

\bibitem[WZL{\etalchar{+}}19]{wilson2019differentially}
Royce~J Wilson, Celia~Yuxin Zhang, William Lam, Damien Desfontaines, Daniel
  Simmons-Marengo, and Bryant Gipson.
\newblock Differentially private sql with bounded user contribution, 2019.

\end{thebibliography}

\newpage

\appendix
\onecolumn
\section{Private Counters}
We use the binary mechanism of \citep{chan2011private} and
\cite{dwork2010differential}
to keep track of important events in a differentially private way.
\begin{algorithm}[h]
\caption{Binary Meachanism $\cB$}
\label{alg:BM}
\KwIn{ Time upper bound $T$, privacy parameter $\epsilon$, stream $\sigma \in \{0,1\}^T$}
$\epsilon' \leftarrow \epsilon / \log T$\\
\For{$t\leftarrow 1$ \textbf{ to } $ T$}{
	Express $t$ in binary form: $t = \sum_j \text{Bin}_j(j) \cdot 2^j$\\
	Let $i := \min\{j: \text{Bin}_j(j)\neq0\}$\\
	$\alpha_i \leftarrow \sum_{j<i}\alpha_j + \sigma(t)$\\
	\For{ $j \longleftarrow 0$ \textbf{to} $i-1$ }{
		$\alpha_j\leftarrow 0, \hat{\alpha}_j \leftarrow 0$
	}
	$\hat{\alpha_i} \leftarrow \alpha_i + \text{Lap}\left(\frac{1}{\epsilon'}\right)$\\
	Output at time $t$ $\cB(t) \leftarrow \sum_{j:\text{Bin}_j(T)=1} \hat{\alpha}_j$\\
}
\end{algorithm}
The error of the counter is given by the following theorem:
%
%
\begin{theorem}[Theorem 4.1 in \cite{dwork2010differential} ]
\label{thm:counterAccuracy}
The counter algorithm \ref{alg:BM} run with parameters $T, \eps, \beta$,
yields a $T$-bounded counter
with $\eps$-differential privacy,
such that with probability
at least $1-\beta$ the error for all prefixes
 $1\leq t \leq T$ is at most
$\tfrac{4}{\eps}\log(1/\beta)\log^{2.5}\left(T\right)$.
\end{theorem}
\section{PAC and Regret Analysis of algorithm \texttt{PUCB}}
In this section we provide the complete PAC and Regret analysis of
algorithm \PUCB~ corresponding to theorem \ref{thm:PUCBPAC} and
\ref{thm:PUCBRegret} respectively.
We begin by analyzing the PAC sample complexity.

\subsection{PAC guarantee for \PUCB. Proof of theorem \ref{thm:PUCBPAC}}

We restate the PAC guarantee.
\begin{theorem*}[PAC guarantee for $\PUCB$. Theorem \ref{thm:PUCBPAC}]
Let $\Tmax$ be the maximum number of episodes and $\eps$ the JDP
parameter. Then for any $\alpha \in (0,H]$ and $\beta \in (0,1)$, algorithm $\PUCB$ with parameters $(\eps,\beta)$ follows a policy that with
probability at least $1-\beta$ is $\alpha$-optimal on all but
\begin{align*}
O\left(\left(\frac{SAH^4}{\alpha^2} + \frac{S^2AH^4}{\eps\alpha}\right)
\polylog\left(\Tmax,S,A,H,\tfrac{1}{\alpha},\tfrac{1}{\beta}, \tfrac{1}{\eps}\right)\right)
\end{align*}
episodes.
\end{theorem*}
The term  $\frac{S^2AH^4}{\eps\alpha}$  in theorem \ref{thm:PUCBPAC} is the extra sample complexity
due to the constraint of differential privacy. Importantly, as we will show in section \ref{sec:lower}, the privacy term matches the lower bound
in  $\eps$ and $\alpha$. Although it remains an open problem
whether the dependence on $S^2$ in the lower order term is necessary
for privacy.
\begin{proof}{of theorem \ref{thm:PUCBPAC}.}
We use a similar approach as in \cite{dann2017unifying} which uses the concept
of nice episodes but we modify their definition of nice episodes
to account for the noise the algorithm adds in order to preserve privacy.
Denote by $[\Tmax]$ the set of all episodes where $\Tmax$ is the maximum
number of episodes and  $\sah \eqdefU (s,a,h)$. The term $w_t(\sah)$
gives the probability of visiting state tuple $\sah$ after following
policy $\pit$ during episode $t$.
Let the set of nice episodes $N\subset [\Tmax]$ be defined as
\begin{align*}
N \coloneqq \left\{ t  : w_t(\sah) < w_{min} \text{ or }
 \frac{1}{4}\sum_{i<t} w_i(\sah) \geq \ln\frac{SAH}{\beta'} + 2\Eeps\right\}
\end{align*}
The number of suboptimal episodes is bounded by the number of suboptimal
nice episodes plus the number of non-nice episodes. In this section
we demonstrate how to bound each individually.

\paragraph{Optimality gap decomposition}
Fix some episode $t$ and let $\pit$ be the policy produced by
algorithm \ref{alg:pucb}. The optimality gap for episode $t$
is denoted by $\Delta_t \eqdefU  V_1^\star - V_1^\pit$.
In section \ref{sec:optimism} we show how to construct the optimistic
$Q$-value function $\QtilUB_t$ used by algorithm \PUCB.
And in section \ref{subsec:optimalitygap} we use the optimism of
$\QtilUB_t$ to decompose the optimality gap as follows:
\begin{align}\label{eq:pac:optgapone}
\Delta_t \leq \sum_{\sah \in\states\times\actions\times[H]}w_t(\sah)
\conftilt(\sah)
\end{align}
We define a set $L_{t}\coloneqq \{\sah: w_t(\sah) < w_{min} \}$
where $w_{min}\coloneqq \frac{\alpha}{3SH^2}$.
The set $L_{t}$ contains all state tuples with low probability of being
visited during episode $t$ by following policy $\pit$.
We can now decompose equation \eqref{eq:pac:optgapone} further
%
\begin{align*}
\Delta_t
\leq
\sum_{\sah \in L_{t}}w_t(\sah)
\conftilt(\sah)
+
\sum_{\sah \notin L_{t}}w_t(\sah)
\conftilt(\sah)
\end{align*}
We choose $w_{min}$ such that $\sum_{\sah \in L_{t}}w_t(\sah)
\conftilt(\sah) \leq \frac{\alpha}{3}$. Hence the gap is upper bounded by:
\begin{align*}
\Delta_t \leq
\tfrac{\alpha}{3}+ \sum_{\sah \notin L_{t}}w_t(\sah) \conftilt(\sah)
\end{align*}
Now we only need to
bound the number of episodes where the term $\sum_{\sah \notin L_{t}}w_t(\sah) \conftilt(\sah)$ is greater
than $2\alpha/3$.
\paragraph{Bounding suboptimal episodes}
First we bound the number of suboptimal nice episodes.
Note that from algorithm \PUCB~ we have
$\conftilt(\sah) = (H+1)\phitilconfL{\sah}+\psitilconfL{\sah}$
if $\ntilt(\sah) \geq 2\Eeps$ otherwise
$\conftilt(\sah) = H$.
However, we use a properties of nice episodes
(from lemma \ref{lem:niceproperties}) which says that
if $\sah\notin L_t$ and $t$ is a nice episode
then $\ntilt(\sah) \geq 2\Eeps$. Therefore, if $t$ is a nice episode,
we can replace every $\conftilt(\sah)$ term
and the gap can be upper bounded by
\begin{align*}
\Delta_t \leq \frac{\alpha}{3} +
\Delta_{1,t} + \Delta_{2,t}
\end{align*}
where
\begin{align*}
\Delta_{1,t}= \sum_{\sah\notin L_t}w_t(\sah)(H+1)\phitilconfL{\sah}
\quad\quad\text{ and }\quad\quad \Delta_{2,t} =\sum_{\sah\notin L_t}w_t(\sah)\psitilconfL{\sah}
\end{align*}

%
%
%
Now we bound the number of nice episodes where term $\Delta_{2,t}$
is greater than $\alpha/3$. Recal that
$\psitilconfL{\sah} = \psitilconfR{\sah}$.
We use lemma \ref{lem:niceproperties} again that says that
w.p at leat $1-\beta$, if $t$ is a nice episode
and if $\sah\notin L_t$ then we have
 $\ntilt(\sah) \geq 2\Eeps$. Thus,
the following is true:
 $\frac{\Eeps^2}{\ntilt(\sah)^2} < \frac{\Eeps}{\ntilt(\sah)}$ on nice
 episode $t$. We can upper bound the gap $\Delta_{2,t}$ with
\begin{align}
\label{eq:gapPsiUpper}
\Delta_{2,t}
  &\leq \sum_{\sah\eqdefU (s,a,h) \notin L_t}w_t(\sah)
 (1+SH)\frac{10\Eeps}{\ntilt(\sah)}
\end{align}
In section \ref{subsec:niceepisodes} we show how to bound
the number of nice episodes where the term from right side of
inequality (\ref{eq:gapPsiUpper}) is bigger than $\alpha/3$.
That is we use lemma \ref{lem:mainrate} from section \ref{subsec:niceepisodes}
with $r=1$ to show that the number of nice episodes $t\in N$
where $\Delta_{2,t}>\tfrac{\alpha}{3}$ is at most
\begin{align}\label{eq:psiBound}
\frac{240\Eeps SAH}{\alpha}\polylog(\Tmax,S,A,H,\tfrac{1}{\eps},\tfrac{1}{\beta},
 \tfrac{1}{\alpha})
\end{align}
%
%


\noindent For gap $\Delta_{1,t}$ we have the upper bound
\begin{align}
\label{eq:gapPhiUpper}
  \Delta_{1,t} \leq \sum_{x\notin L_t} w_t(\sah) (H+1)\phitilconfR{\sah}
\end{align}
We use lemma \ref{lem:mainrate} from section \ref{subsec:niceepisodes}
again with $r=1/2$ to show that
the right side of equation \ref{eq:gapPhiUpper} is greater than
 $\alpha/3$ on at most
\begin{align}
\label{eq:phiBound}
\frac{18SAH^4}{\alpha^2}\text{polylog}(\Tmax,S,A,H,\tfrac{1}{\beta}, \tfrac{1}{\alpha})
\end{align}
nice episodes.
Finally, lemma \ref{lem:notnice} from \ref{subsec:niceepisodes} says
that the set of non-nice episodes is at most
\begin{align}\label{eq:notnice}
\frac{120S^2AH^4}{\alpha\eps}\polylog(\Tmax,S,A,H, \tfrac{1}{\beta})
\end{align}
%
Combining equations \ref{eq:notnice}, \ref{eq:phiBound}, and
\ref{eq:psiBound} gives the most number of $\alpha$-suboptimal
episodes
\begin{align*}
O\left(
\left(\frac{S^2AH^4}{\eps\alpha} +\frac{SAH^3}{\alpha^2} \right) \polylog(\Tmax,S,A,H,\tfrac{1}{\beta}, \tfrac{1}{\alpha}) \right)
\end{align*}
conmpleting the proof.

\end{proof}

%
\subsection{Regret bound for \PUCB. Proof of theorem \ref{thm:PUCBRegret}}
In this section we layout the proof the regret bound from theorem \ref{thm:PUCBRegret}. We reuse some tools from the previous PAC analysis
 and also use similar techniques as in \citep{azar2017minimax}.
 As in the PAC analysis the key to getting the right dependence on
 $\alpha$ and $\eps$ lies in the decomposition of the confidence bounds.
We restate the theorem below and the provide the proof.
\begin{theorem*}[Regret bound for \PUCB. Theorem \ref{thm:PUCBRegret}.]
With probability at least $1-\beta$, the regret of $\PUCB$ up to episode $T$ is at most
\begin{align*}
  O\left(\left(H^2\sqrt{SAT} + \frac{SAH^3 + S^2AH^3}{\eps}\right)
   \polylog\left(T, S, A, H, \tfrac{1}{\beta}, \tfrac{1}{\eps}\right)\right) \enspace.
\end{align*}
\end{theorem*}
\begin{proof}{of theorem \ref{thm:PUCBRegret}}
Denote by $[\Tmax]$ the set of all episodes where $\Tmax$ is the maximum
number of episodes. Let  $\sah \coloneqq (s,a,h)$. The term $w_t(\sah)$
gives the probability of visiting state tuple $\sah$ after following
policy $\pit$ during episode $t$.

Let $\Delta_t \coloneqq  V_1^\star - V_1^\pit$ be the optimality gap
for episode $t$ given that
the learner plays policy $\pit$, then the expected regret of the learner at
episode $T\in[\Tmax]$ is given by
\begin{align}\label{eq:main:regretI}
\regret(T) = \sum_{t=1}^T	 \Delta_t
\end{align}

\paragraph{Optimality gap decomposition}
In section \ref{sec:optimism} we show how to construct the optimistic
$Q$-value function $\QtilUB_t$ used by algorithm \PUCB.
And in section \ref{subsec:optimalitygap} we use the optimism of
$\QtilUB_t$ to decompose the optimality gap as follows:
\begin{align}\label{eq:optgapone}
\Delta_t \leq
\sum_{h=1}^H \bE_{s \sim \pi_t(\cdot, h)}\conftilt(s,\pi_t(s,h), h)
\end{align}
%
Therefore the regret is bounded by
\begin{align}\label{eq:main:regretI}
\regret(T) \leq \sum_{t=1}^T
\sum_{h=1}^H
\ex{s \sim \pi_t(\cdot, h)}{\conftilt(s,\pit(s,h), h)}
\end{align}
For brevity let $\sah_{t,h}\coloneqq (s_{t,h}, \pit(s_{t,h}), h)$
where $s_{t,h}$ is the state visited by the agent during episode $t$
and time $h$.
Then we can bound the regret  by
%
\begin{align}
\label{eq:regret:gapI}
\regret(T)
\leq \sum_{t=1}^T \sum_{h=1}^H
\left(
\ex{s \sim \pi_t(\cdot, h)}{\conftilt(s,\pit(s,h),h)}-\conftilt(\sah_{t,h})
\right)
+ \sum_{t=1}^T \sum_{h=1}^H \conftilt(\sah_{t,h})
\end{align}
Next step to get the regret bound is to bound each term
from equation \eqref{eq:regret:gapI} individually.
\paragraph{Bounding martingale sequence $\sum_{t=1}^T \sum_{h=1}^H \ex{s \sim \pi_t(\cdot, h)}{\conftilt(s,\pit(s,h),h)}-\conftilt(\sah_{t,h})$: }
The first of equation (\ref{eq:regret:gapI}) is sequence of random variables.
Azuma's concentration bound says that
$\pr{X_n - X_0 > b } < \exp\left(\frac{-2b^2}{\sum_{i=1}^n c_i^2}\right)$ for martingale sequence $(X_{i})$ such that $|X_i - X_{i-1}| < c_i$.
Let
\begin{align*}
X_t = \sum_{i=1}^t \sum_{h=1}^H
\left(
\ex{s \sim \pi_i(\cdot, h)}{\conftil_i(s,\pi_i(s,h),h)}-\conftil_i(\sah_{i,h})
\right)\end{align*}
be a sequence of $T$ random variables where each $X_t$ depends on the
realizations of the previous $X_1, \ldots, X_{t-1}$. Then it follows
by the boundness of $\conftil(\cdot)$ that
$|X_t - X_{t-1}| \leq H^2$ and that each random variable $X_t$ in the sequence
has mean zero. Hence $X_1, \ldots, X_T$ is a martingale sequence and
we can apply Azuma's inequality to get that on round $T$ with probability
at least $1-\beta'$ we have
\begin{align}
\label{eq:firstBound}
\sum_{t=1}^T \sum_{h=1}^H \left( \ex{s \sim \pi_t(\cdot, h)}{\conftilt(s,\pit(s,h),h)}-\conftilt(\sah_{t,h}) \right)
\leq  H^2\sqrt{\tfrac{1}{2}T \log(1/\beta')}
\end{align}
The last step is to set fail probability to $\beta' = \tfrac{\beta}{\Tmax}$ and
apply union bound over $\Tmax$ rounds

\paragraph{Bounding Exploration bonus term $\sum_{t=1}^T \sum_{h=1}^H \conftilt(\sah_{t,h})$:}
Next, we focus on bounding the second term from equation \eqref{eq:regret:gapI}. As seen before $(x_{t,1},\ldots,x_{t,H})$ the a sequence of state-tuples corresponding to the trajectory observed by the agent during episode $t$. Let $N$ be the set of episodes with each state-tuple in the trajectory visited at least $3\Eeps$ many times, that is,
\[
N\eqdefU \left\{ t\in [\Tmax] :\quad \forall_{ h\in [H]},  \nhatt(\sah_{t,h})\geq 3\Eeps \right\}
\]
Then we can decompose the second term from equation \eqref{eq:regret:gapI} using $N$ as follows:
%
\begin{align}
\label{eq:regret:decomposed}
\sum_{t=1}^T \sum_{h=1}^H \conftilt(\sah_{t,h})
= \sum_{t\in N} \sum_{h=1}^H \conftilt(\sah_{t,h})
+ \sum_{t\notin N} \sum_{h=1}^H \conftilt(\sah_{t,h})
\end{align}
There is only a finite number of episodes with $t\notin N$. To bound the maximum cardinality of the set $\{t: t\notin{N}\}$, consider  the smallest visitation count in the trajectory of episode $t$, let's denote it by $m_t = \min_{h} \nhatt(\sah_{t,h})$.

By a pigeon-hole argument, after $SA$ episodes $m_t$ must increase by at least one. It follows that after $3\Eeps SA$ episodes we have that $m_t \geq 3\Eeps $. Hence $\left|\{t: t\notin{N} \}\right|\leq 3\Eeps SA$.
Therefore, we can bound the second term of equation \eqref{eq:regret:decomposed}, $\sum_{t\notin N} \sum_{h=1}^H \conftilt(\sah_{t,h})$, by
\begin{align}
\label{eq:regret:loworderterm}
\sum_{t\notin N} \sum_{h=1}^H \conftilt(\sah_{t,h})
\leq 3\Eeps SAH \conftilt(\sah_{t,h})
\leq 3\Eeps SAH^2
= \tfrac{3}{\eps} SAH^3 \log\left(\tfrac{2SAH+S^2AH}{\beta'}\right)
\log(\Tmax)^{5/2}
\end{align}
The first inequality of \eqref{eq:regret:loworderterm} follows from $\left|\{t: t\notin{N} \}\right|\leq 3\Eeps SA$, and the second inequality from $\conftilt(\sah_{t,h})\leq H$. We get the last equality by setting $\Eeps  = \counterBound$.
Next note that
if $t \in N$ then visited states $\sah_{t,h}$
on episode $t$ have been seen at least $3\Eeps$ many times, i.e,
$\nhatt(\sah_{t,h})\geq 3\Eeps $. Then on the high probability event
$| \ntilt(\sah_{t,h}) - \nhatt(\sah_{t,h})| \leq \Eeps$ it follows that
$\ntilt(\sah_{t,h}) \geq 2\Eeps$ hence, from algorithm
 \ref{alg:privateplanning}, we have that
$\conftilt(\sah_{t,h}) = \psitilconfL{\sah_{t,h}} + (H+1)\phitilconfL{\sah_{t,h}}$.  Now we can decompose
the first term of \ref{eq:regret:decomposed} as
\begin{align}
\label{eq:regret:gapIII}
\sum_{t\in N} \sum_{h=1}^H \conftilt(\sah_{t,h})
= (H+1)\sum_{t\in N} \sum_{h=1}^H \phitilconfL{\sah_{t,h}}
+  \sum_{t\in N} \sum_{h=1}^H \psitilconfL{\sah_{t,h}}
\end{align}
We now bound the first term of equation \eqref{eq:regret:gapIII}.
We will use a pigeon-hole argument for the next step which goes
as follows.
If for all state tuples $\sah_{t,h}$ in the trajectory
of the agent we have $\nhatt(\sah_{t,h})\geq 1$ then
\begin{align*}
  \sum_{t=1}^T \sum_{h=1}^H \frac{1}{\nhatt(\sah_{t,h})}
  \leq  \sum_{(s,a,h)\in \states\times\actions\times[H]} \sum_{m = 1}^{\nhatt(s,a,h)}\tfrac{1}{m}
  \leq SAH \ln(T)
\end{align*}
The last inequality follows from the fact that $\nhatt(\cdot)\leq T$
and the bound
$\sum_{m=1}^T = \frac{1}{1} + \frac{1}{2} + \ldots + \frac{1}{T} \leq \ln(T) + 1$.
Recall that
\[
\phitilconfL{\sah_{t,h}} \coloneqq \phitilconfR{\sah_{t,h}}
\]
and let $L = 2\ln\left(\Tmax / \beta'\right)$.
We are ready to bound the first term of equation \eqref{eq:regret:gapIII}
\begin{align}
\notag
(H+1)\sum_{t\in N} \sum_{h=1}^H \phitilconfL{\sah_{t,h}}
&\leq
(H+1)\sqrt{L}\sum_{t\in N}\sum_{h=1}^H\sqrt{\frac{1}{\ntilt(\sah_{t,h})-\Eeps}}\\
\label{eq:phI}
&\leq
(H+1)\sqrt{L}\sum_{t\in N}\sum_{h=1}^H\sqrt{\frac{1}{\nhatt(\sah_{t,h})-2\Eeps}}\\
\label{eq:CS}
&\leq
(H+1)\sqrt{L}
\sqrt{\sum_{t\in N}\sum_{h=1}^H1}
\sqrt{\sum_{t\in N}\sum_{h=1}^H\frac{1}{\nhatt(\sah_{t,h})-2\Eeps}} \\
\label{eq:PH}
&\leq
(H+1)\sqrt{L}\sqrt{TH}
\sqrt{\sum_{\sah\in \states\times\actions\times[H]}
\sum_{m=\Eeps}^{\nhatt(\sah)-2\Eeps}\frac{1}{m}} \\
\label{eq:seqBound}
&\leq
(H+1)\sqrt{L}\sqrt{TH}
\sqrt{SAH \ln(T)} \\
\label{eq:secondBound}
&\leq H^2 \sqrt{SA T} \polylog(\Tmax, S, A, H, 1/\beta)
\end{align}

For the inequality \eqref{eq:phI}  we use fact that on the good event
we have $\ntilt(\sah_{t,h})\geq \nhatt(\sah_{t,h})-\Eeps$. For
inequality \eqref{eq:CS} we use Cauchy-Schwarz inequality. Then we
use the pigeon-hole principle for inequality \eqref{eq:PH}.
Finally for inequality \eqref{eq:seqBound} we use the bound
$\frac{1}{1} + \frac{1}{2} + \ldots + \frac{1}{T} \leq \ln(T) + 1$

Next we bound the second term of equation \eqref{eq:regret:gapIII}. Recall that
\[
\psitilconfL{\sah_{t,h}} \coloneqq \psitilconfR{\sah_{t,h}}
\]
Since we are considering episodes in $t\in N$  for each
$\sah_{t,h}$ we have that $\nhat(\sah_{t,h})\geq 3\Eeps$ and thus
$\ntilt(\sah_{t,h})\geq 2\Eeps$ on the high-probability good event.
Then it follows that we can upper bound $\frac{\Eeps^2}{\ntilt(\sah_{t,k})^2}$
by $\frac{\Eeps}{\ntilt(\sah_{t,k})}$. We end up with
\begin{align}
\notag
\sum_{t\in N} \sum_{h=1}^H \psitilconfL{\sah_{t,h}}
&\leq
(1+SH)5\Eeps\sum_{t\in N}\sum_{h=1}^H\frac{1}{\ntilt(\sah_{t,h})}\\
&\leq
(1+SH)5\Eeps\sum_{t\in N}\sum_{h=1}^H\frac{1}{\nhatt(\sah_{t,h})-\Eeps}\\
&\leq
(1+SH)5\Eeps\sum_{\sah\in\states\times\actions\times[H]}
\sum_{m=2\Eeps}^{\nhatt(\sah)-\Eeps}\frac{1}{m} \\
\notag
&\leq 5S^2AH^2\Eeps\log(T) \\
\label{eq:thirdBound}
&\leq \frac{15}{\eps} S^2 A H^3\log(\tfrac{1}{\beta'}) \log(T)^{7/2}\polylog(S,A,H)
 \end{align}
Putting together inequalities \eqref{eq:firstBound}, \eqref{eq:regret:loworderterm} , \eqref{eq:secondBound}, \eqref{eq:thirdBound} we have
\begin{align*}
\regret(T) \leq
H^2\sqrt{\tfrac{1}{2}T \log(T/\beta)}
+ H^2 \sqrt{SA T} \polylog(T, S, A, H, 1/\beta) \\
+ \frac{15}{\eps} S^2 A H^3\log(\tfrac{1}{\beta'}) \log(T)^{7/2}\polylog(S,A,H)\\
+ \tfrac{3}{\eps} SAH^3 \log\left(\tfrac{2SAH+S^2AH}{\beta'}\right)
\log(T)^{5/2}
\end{align*}

\end{proof}

\subsection{Error bounds}

The proof of theorem \ref{thm:PUCBPAC} relies on
the confidence bounds
on the empirical estimates of the
sufficient statistics $\nhatt(s,a,h)$, $\rhatt(s,a,h)$ and
$\mhatt(s,a,h)$ as well as the error upper bound
on the private estimates $\ntilt(s,a,h)$, $\rtilt(s,a,h)$ $\mtilt(s,a,h)$.

We will combine the confidence bounds from  definition \ref{def:failevents}
to construct a confidence interval $\conftil(s,a,h)$ for the
private $Q$-values
The main challenge lie in getting the correct sample complexity dependence
on the target accuracy $\alpha$ and the privacy parameter $\eps$.

\begin{definition}\label{def:failevents}
The fail event
\begin{align*}
F =
\bigcup_{t=1}^\Tmax \left[ F_t^N \cup F_t^R \cup  F_t^V
\cup F_t^{PR} \cup F_t^{PN} \cup F_t^{PM}
\right]
\end{align*}
\noindent where $\Tmax$ is the maximum number of episodes is defined as
\begin{align*}
F_t^N &= \bigg\{ \exists s,a,h: \nhat_t(s,a,h) <
\frac{1}{2} \sum_{i<t} w_{i,h}(s,a) -\ln\frac{SAH}{\beta'}\bigg\} \\
F_t^R &= \Bigg\{
\exists s,a,t : \left| \frac{\rhatt(s,a,h)}{\nhatt(s,a,h)} - r(s,a,h)\right|
\geq \phihatconfL{s,a,h}
\Bigg\}\\
F_t^V &= \bigg\{
\exists{s,a,h} : |(\widehat{\cP}_t(s,a,h) - \cP(s,a,h))^{\top} V_{h+1}^*|\geq H\phihatconfL{s,a,h}
\bigg\}\\
F_t^{PR} &= \left\{
\exists{s,a,h} : |\rtil_t(s,a,h) - \rhat_t(s,a,h)| \geq \Eeps
\right\}\\
F_t^{PN} &= \left\{
\exists{s,a,h} : |\ntil_t(s,a,h) - \nhat_t(s,a,h)| \geq \Eeps
\right\}\\
F_t^{PM} &= \{
\exists{s,a,s',h} : |\mtil_t(s,a,s',h)
- \mhat_t(s,a,s',h)| \geq \Eeps
\}
\end{align*}
where $\phihatconfL{s,a,h} = \phihatconfR{s,a,h}$ and
$\Eeps = \counterBound$.
\end{definition}
\begin{lemma}
\label{lem:faileventF}
The fail event $F$ from definition \ref{def:failevents} occurs
with probability at most $\beta$.
\end{lemma}
\begin{proof}
From \cite[Corollary E.4]{dann2017unifying} we have that
   $\pr{\bigcup_{t=1}^T F_t^N}\leq \beta'$. Using the standard Chernoff Bound
   inequality and union bound over $T$ rounds we have that
    $\pr{ \bigcup_{t=1}^T F_t^R} \leq \beta'$ and
    $\pr{ \bigcup_{t=1}^T F_t^V} \leq \beta'$.

There is a total of $SAH + S^2AH$ private counters.
From the error bound (theorem \ref{thm:counterAccuracy}) of
the private counter from algorithm
\ref{alg:BM} with fail probability set to
$\tfrac{\beta'}{SAH + S^2AH}$
and applying union bound over all $SAH + S^2AH$ counters
we have
\[
\pr{ \bigcup_{t=1}^T  \left( F_t^{PR} \cup F_t^{PN} \cup
F_t^{PM}  \right)  } \leq \beta'
\]
Finally setting $\beta' = \tfrac{\beta}{4}$ and applying another
union bound over $4$ fail events we obtain that $\pr{F} \leq \beta$.
\end{proof}
%

%
\subsection{$Q$-optimism}\label{sec:optimism}
The proof will follow the principle of optimism under uncertainty as in
previous work \cite{dann2017unifying}.
In order to obtain the right sample complexity dependence
on $\alpha$ and $\eps$
we must construct a confidence bound  that disentangles
the sampling error from the
empirical estimates and the error from the private counters.

We will construct a confidence bound
for any  $(s,a,h)\in\states\times\actions\times[H]$ and any
episode $t\in [T]$.
To reduce notation clutter we will use $\sah\eqdefU (s,a,h)$
in place of $(s,a,h)$.
Our objective in this section is to use the private counts
to construct an optimistic
and private $Q$-function  $\QtilUB_t$.
We say that a $Q$-function $\QhatUB$ is optimistic with respect to
the $Q$-function induced from the optimal policy $Q^*$ if
for all $\sah$ we have $\Qhat_t(\sah) \geq Q^*(\sah)$ with high
probability.
To that end we first construct the optimistic but non-private
$\QhatUB_t$ using the
non-private counters. Let us first define the empirical $Q$-value
estimate $\Qhat_t$ as
\begin{align*}
	\Qhat_t(\sah) =
	\frac{\rhatt(\sah)
		+ \sum_{s'\in\states}\Vtil_{h+1}(s')\mhatt(\sah,s')}{\nhatt(\sah)}
\end{align*}
and the optimistic $Q$-function on episode $t$ is by
\begin{align}\label{eq:Qhatdef}
	\QhatUB_t(\sah) = \Qhat_t(\sah)
		+ (H+1)\phihatconfL{\sah}
\end{align}
where $\phihatconfL{\sah} = \phihatconfR{\sah}$.

To show that $\QhatUB_t(\sah)$ is optimistic with respect to $Q^*(\sah)$
it suffices to show that $\Vtil_h$ is optimistic with respect to
$V^*$ which follows from induction on $h\in [H]$ and from the standard concentration bounds. These two
requirements are formilized in lemma \ref{lem:Voptimism} and
\ref{lem:Qhatoptimism} below.

To construct a private $Q$-function  we must use the
private counts.
We denote the private empirical $Q$-value estimate on episode $t$ by
\begin{align*}
&\Qtil_t(\sah) =
\frac{\rtilt(\sah)
+\sum_{s'\in\states}\Vtil_{h+1}(s')\mtilt(\sah,s')}{\ntilt(\sah)}\\
\end{align*}
The optimistic private $Q$-function is defined as
\begin{align}\label{eq:Qtildef}
\QtilUB_t(\sah) = \Qtil_t(\sah) + (H+1)\phitilconfL{\sah}
+\psitilconfL{\sah}
\end{align}

\begin{theorem}
On the good event $F^c$,
the $Q$-function $\QtilUB_t$  from equation \ref{eq:Qtildef}
is optimistic. That is, for any $t\in[T]$ we have $\QtilUB_t(\sah)\geq \QhatUB_t(\sah)$  for all tuples
$\sah \eqdefU (s,a,h) \in \states\times\actions\times[H]$.
\end{theorem}
\begin{proof}
First use lemma \ref{lem:Qtiloptimism} to show that the private optimismtic
$\QtilUB$ is optimistic with respect to $\QhatUB$ from equation
\ref{eq:Qhatdef}. From lemma \ref{lem:Qhatoptimism} we have that the $Q$-function $\QhatUB$ is optimistic. Putting it all together, on the
good event $F^c$, for any round $t$ and all state tuples
 $\sah \eqdefU (s,a,h)$ we have
\begin{align*}
  \QtilUB_t(\sah) \geq \QhatUB_t(\sah) \geq Q^*(\sah)
\end{align*}
completing the proof.
\end{proof}

\begin{lemma}\label{lem:Qtiloptimism}
On the good event $F^c$, for any $t\in[T]$ we have $\QtilUB_t(\sah)\geq \QhatUB_t(\sah)$  for all tuples
$\sah \eqdefU (s,a,h) \in \states\times\actions\times[H]$.
\end{lemma}
\begin{proof}
We attempt to construct a confidence bound
for $\Qtil_t$. By using the  error bound $\Eeps$ of the private counters,
we can upper bound $\Qhat_t$ in terms of the private counters as follows
\begin{align}
\label{eq:QhatUpper}
&\Qhat_t(\sah)
\leq\frac{\rtilt(\sah)+\Eeps
+\sum_{s'\in\states}\Vtil_{h+1}(s')(\mtilt(\sah,s')+\Eeps)}{\ntilt(\sah)-\Eeps}
\end{align}
The following claim will help us recover $\Qtil_t$ from \ref{eq:QhatUpper}.
\begin{claim*}{\ref{claim:main:oneoverNbound}}
Let $y \in \mathbb{R}$ be any positive real number.
Then for all $x \in \mathbbm{R}$ with $x\geq 2y$ it holds that
$\frac{1}{x - y} \leq \frac{1}{x} + \frac{2y}{x^2}$
\end{claim*}
To get $\Qtil_t$ back from inequality  \ref{eq:QhatUpper} we apply
claim \ref{claim:main:oneoverNbound} which allows us write
$\frac{1}{\nhatt(\sah) - \Eeps} \leq
\frac{1}{\ntilt(\sah)}  + \frac{2\Eeps}{\ntilt(\sah)}$
when $\ntilt(\sah)\geq 2\Eeps$.
Then
\begin{align*}
\Qhat_t(\sah)
\leq	&\left( \frac{1}{\ntilt(\sah)} + \frac{2\Eeps}{\ntilt(\sah)^2} \right)
\left( \rtilt(\sah) + \Eeps + \sum_{s'\in\states} \Vtil_{h+1}(s')(\mtilt(\sah, s') + \Eeps) \right) \\
& \leq \frac{\rtilt(\sah)+ \sum_{s'\in\states} \Vtil_{h+1}(s')\mtilt(\sah,s')}{\ntilt(\sah)}
 + 3\frac{\Eeps + SH\Eeps}{\ntilt(\sah)}
+2\frac{\Eeps^2  + SH\Eeps^2}{\ntilt(\sah)^2}  \\
& = \Qtil_t(\sah) + \psitilconfL{\sah}
\end{align*}
So far we have the following bound when $\ntilt(x)\geq 2\Eeps$
\begin{align*}
Q^*(\sah)\leq \QhatUB_t(\sah)\leq \Qtil_t(\sah) + \psitilconfL{\sah}  + (H+1)\phihatconfL{\sah}
\end{align*}
In the case when $\ntilt(x)<2\Eeps$ we can simly upper bound
$Q^*$ by $H$.
The last step is to replace the term $\phihatconfL{\sah}$ which is not
a private object. We again use the error bound $\Eeps$ from the private
counters to write
\begin{align*}
\phihatconfL{\sah} = \phihatconfR{\sah} 	\leq
\phitilconfR{\sah} = \phitilconfL{\sah}
\end{align*}

Finally we can write an expression for the optimistic private
$Q$-function as
\[
\QtilUB(\sah) \eqdefU \\
\begin{cases}
\Qtil(\sah) + 	\psitilconfL{\sah} + (H+1)\phitilconfL{\sah} &
\text{ if } \ntilt(\sah) \geq 2\Eeps \\
H & \text{ otherwise}
\end{cases}
\]
Therefore we have by  construction that $\QtilUB(\sah) \geq \QhatUB(\sah)$.
\end{proof}
\begin{lemma}[Value function optimism]\label{lem:Voptimism}
	On the event $F^c$,  the
	value function from algorithm \ref{alg:privateplanning} is optimismtic.
	i.e. for all $s\in\states$ and all $h\in[H]$ we have
	$\Vtil_h(s) \geq V^*_{h}(s)$
\end{lemma}
%
%
\begin{lemma}\label{lem:Qhatoptimism}
On the good event $F^c$,
$\QhatUB_t(\sah)$ is optimistic with respect to $Q^*(\sah)$
\end{lemma}
%
\begin{proof}
Let the empirical mean reward on round $t$ be $\bar{r}_t(\sah) = \frac{\rhatt(\sah)}{\nhatt(\sah)}$ and let $\widehat{P}_t(\sah) \Vtil_{h+1} = \frac{\sum_{s'\in\states}\Vtil_{h+1}(s')\mhatt(\sah, s')}{\nhatt(\sah)}$. Now we can write $\QhatUB_t$ as
\begin{align*}
\QhatUB_t(\sah) = \bar{r}_t(\sah) + \widehat{P}_t(\sah)\Vtil_{h+1} + (H+1)\phihatconfL{\sah}
\end{align*}
\noindent On the event $F^c$ we have that
$\frac{\rhatt(\sah)}{\nhatt(\sah)} - r(\sah) \leq \phihatconfL{\sah}$ and
$(\widehat{P}_t(\sah) - P(\sah) )V^*_{h+1} \leq H \phihatconfL{\sah}$, where $P(\sah)$ is true next state distribution.
 Furthermore, from
the value function optimism lemma (\ref{lem:Voptimism})
we have that for all $s \in \states$, and all $h\in[H]$,
$\Vtil_h(s) \geq V_h^*(s)$. Putting it all together:
\begin{align*}
\QhatUB_t(\sah) - Q^*(\sah)
& = \bar{r}_t(\sah) - r(\sah) + \widehat{P}_t(\sah)\Vtil_{h+1} - P(\sah) V^*_{h+1}
+ (H+1)\phihatconfL{\sah} \\
&\geq \bar{r}(\sah) - r(\sah) + (\widehat{P}(\sah) - P(\sah) )V^*_{h+1}
+ \phihatconfL{\sah} + H\phihatconfL{\sah} \\
&\geq 0
\end{align*}
Completing the proof.
\end{proof}
%
%
\paragraph{Proof of lemma \ref{lem:Voptimism}}
\begin{proof}
The proof proceeds by induction.
Fixing any state $s\in\states$, we must show that $\Vtil_h(s)\geq V^*_h(s)$
for all $h\in[H]$. For the base case, note that $\Vtil_{H+1}(s)=V^*_h(s)=0$.
 Now assume that for any $h\leq H$, $\Vtil_{h+1}(s) \geq V_{h+1}^*(s)$.
Then we must show that $\Vtil_h(s) \geq V_h^*(s)$.
First write out the equation for $V^*_h(s)$:
\begin{align}\label{eq:vstar}
	V^*_h(s) = \max_{a^*\in\actions} \left(
		r(s,a^*,h)  + P(s,a^*,h)V^*_{h+1}
	\right)
\end{align}
where $r(s,a,h)$ is the true mean reward and $P(s,a,h)$ the true
transition function for state tuple $(s,a,h)$.
Let $a^*$ be the action corresponding to equation ($\ref{eq:vstar}$).
Next we write out the equation for $\Vtil_h(s)$
\begin{align*}
	\Vtil_h(s)
	&= \max_{a\in\actions}
		\frac{\rtilt(s,a,h)+ \sum_{s'\in\states} \Vtil_{h+1}\mtil_t(s,a,s',h)}{\ntilt(s,a,h)}
		 + \conftilt(s,a,h)  \\
	&\geq \frac{\rtilt(s,a^*,h)+ \sum_{s'\in\states} \Vtil_{h+1}\mtilt(s,a^*,s',h)}{\ntilt(s,a^*,h)} + \conftilt(s,a^*,h) \\
	& = \QtilUB_t(s,a^*,h)
\end{align*}
Next we use lemma \ref{lem:Qtiloptimism} which says that on event $F^c$,
$\QtilUB_t(s,a^*,h)\geq \QhatUB_t(s,a^*,h)$
to get a lower bound for $\Vtil_h$:
\begin{align*}
	\Vtil_h(s)
\geq \frac{\rhatt(s,a^*,h) + \sum_{s'\in\states} \Vtil_h(s')\mhatt(s,a^*,s', h)}{\nhatt(s,a^*,h)}
+ (H+1)\phihatconfL{s,a^*,h}
\end{align*}
Letting
\begin{align*}
\frac{\sum_{s'\in\states} \Vtil_h(s')\mhatt(s,a^*,s', h)}{\nhatt(s,a^*,h)} = \widehat{P}(s,a^*,h)\Vtil_{h+1}
\end{align*}
and applying the inductive step (i.e. $\forall_s \Vtil_{h+1}(s)\geq V^*_{h+1}(s)$) we get
\begin{align*}
	&\Vtil_h(s) \geq \frac{\rhatt(s,a^*,h)}{\nhatt(s,a^*,h)}  +
		\widehat{P}(s,a^*,h)V^*_{h+1} + (H+1)\phihatconfL{s,a^*,h}
\end{align*}
Next we use the concentration bound from definition \ref{def:failevents}
\begin{align*}
\widehat{P}(s,a^*,h)V^*_{h+1}\geq P(s,a^*,h)V^*_{h+1} - H\phihatconfL{s,a^*,h}
\end{align*}
and then we use the definition of $V^*_h(s)$
\begin{align*}
P(s,a^*,h)V^*_{h+1} = - r(s,a^*, h) +V^*_{h}(s)
\end{align*}
to get
\begin{align*}
	\Vtil_h(s)
		&\geq \frac{\rtil(s,a^*,h)}{\ntil(s,a^*,h)}  +
		P(s,a^*,h)V^*_{h+1} + \widehat{\phi}(s,a^*, h) \\
		&= \frac{\rtil(s,a^*,h)}{\ntil(s,a^*,h)} - r(s,a^*, h)  + \widehat{\phi}(s,a^*, h) + V^*_{h}(s) \end{align*}
Next note that on event $F^c$, we have
\begin{align*}
\frac{\rtil(s,a^*,h)}{\ntil(s,a^*,h)} - r(s,a^*, h) + \widehat{\phi}(s,a^*, h)> 0
\end{align*}
Hence it follows that $\Vtil_h(s) \geq V_h^*(s)$,
completing the proof.
\end{proof}

%
\subsection{Optimality gap}\label{subsec:optimalitygap}
The next step is
to decompose the optimality gap
$\Delta_t \eqdefU V_1^\star - V_1^\pit$ for episode $t$.
The following lemma states that we can upper bound $\Delta_t$
by the weighted sum of confidence terms:
\begin{lemma}
\label{lem:optgap}
Let $\pit$ be the policy played by algorithm \ref{alg:pucb} during
episode $t$.
Let $w_t(s,a,h)$ be the probability of visiting state tuple $(s,a,h)$
during episode $t$. Then the optimality gap is bounded by
\begin{equation}\begin{aligned}
V_1^\star - V_1^\pit
&\leq \sum_{h=1}^H \bE_{s \sim \pit(\cdot,h)}\conftilt(s,\pit(s,h), h)
= \sum_{(s,a,h) \in\states\times\actions\times[H]}w_t(s,a, h)
\conftilt(s,a, h) \\
\end{aligned}\end{equation}
\end{lemma}
%
\begin{proof}
On episode $t$, let $\QtilUB$ be the private and optimismtic $Q$-fuction
from algorithm \ref{alg:privateplanning}.
Given that the learner is following deterministic policy
$\pit:\states\times[H]\rightarrow \actions$,
then for within episode time-step $h$
we use the short hand notation
$r_h(s)\eqdefU \frac{\rtilt(s,\pit(s,h),h)}{\ntilt(s,\pit(s,h),h)} $
to denote the private mean reward estimate on state $s$ and
$p_h(s, s')\eqdefU\frac{\mtilt(s,\pit(s,h), s', h)}{\ntilt(s,\pit(s,h), h)}$
to denote the transition probability from state $s$ to state $s'$.
Let $s_1, \ldots, s_H$ be random variables, where each $s_h$
represents the state visited
during time-step $h$ after following policy $\pit$. Next let
$\bE_{s_1, \ldots, s_H \sim \pit}$
denote the expectation only over the randomness of the states
$s_1, \ldots, s_H$ after following the deterministic policy $\pit$
on the MDP. For brevity let us use $\bE_\pit$ instead of
$\bE_{s_1, \ldots, s_H \sim \pit}$.
\noindent Then, from the definition of $\QtilUB_t$ in algorithm \ref{alg:privateplanning} we have
\begin{align*}
\bE_{\pit} \QtilUB_t(s_1, \pit(s_1,1), 1)
= \bE_{\pit}r_1(s_1)  +
\bE_{\pit}p(s_1,s_2) \Vtil_2(s_2)
+ \bE_{\pit}\conftilt(s_1,\pit(s_1,1) , 1) \\
\end{align*}
Since $p(s_1,s_2)\leq 1$, if we set $\Vtil_2(s_2) =  \QtilUB_t(s_2, \pit(s_2, 2), 2)$ it follows that
\begin{align}
\label{eq:qtilinequality}
\bE_{\pit} \QtilUB_t(s_1, \pit(s_1,1), 1)
- \bE_{\pit}r_1(s_1)
\leq \bE_{\pit}\QtilUB_t(s_2, \pit(s_2), 2)
+ \bE_{\pit}\conftilt(s_1, \pit(s_1,h), 1)
\end{align}

\noindent Then the optimality gap on episode $t$ is
\begin{align}
\notag
V_1^\star - V_1^\pit
&= V^\star - \sum_{h=1}^H \bE_{\pi_t}r_h(s_h)  \\
\label{eq:bydef}
&= \bE_{\pi_t}  Q^\star(s_1,\pi^
\star(s_1), 1) - \sum_{h=1}^H \bE_{\pi_t}r_h(s_h) \\
\label{eq:optimismInequality}
& (\text{Optimism}) \\
&\leq \bE_{\pi_t} \QtilUB_t(s_1,\pi^*(s_1), 1)
- \sum_{h=1}^H \bE_{\pi_t}r_h(s_h) \\
\label{eq:greedyIinequality}
& (\text{Greedy of policy } \pi_t) \\
&\leq \bE_{\pi_t} \QtilUB_t(s_1,\pit(s_1,1), 1)
- \sum_{h=1}^H \bE_{\pi_t}r_h(s_h) \\
\notag
&\leq \bE_{\pi_t} \QtilUB_t(s_1,\pit(s_1,1), 1)
- \bE_{\pi_t}r_1(s_1) -  \sum_{h=2}^H \bE_{\pi_t}r_h(s_h) \\
\label{eq:inequality}
& (\text{Apply inequality \ref{eq:qtilinequality}}) \\
&\leq \bE_{ \pi_t} \conftilt(s_1,\pit(s_1,1), 1)
	+ \bE_{\pi_t} \QtilUB_t(s_2,\pit(s_2,2), 2)
	 - \sum_{h=2}^H \bE_{\pi_t}r_h(s_h)
\end{align}
The equality \ref{eq:bydef} is from the definition of $V_1^*$
and $V_1^\pit$.
Then inequality \ref{eq:optimismInequality} follows from the optimism of
$\QtilUB_t$. The inequality \ref{eq:greedyIinequality} follows from $\pit$
following a greedy strategy. And finally, inequality \ref{eq:inequality}
is from applying inequality \ref{eq:qtilinequality}.

\noindent Since $\QtilUB_t(s, a, H+1) = 0$ for all state-action tuples
$(s,a)$ then applying equation (\ref{eq:qtilinequality}) $H-1$ more times,
we get
\begin{equation}\begin{aligned}
V_1^\star - V_1^{\pit}
&\leq \sum_{h=1}^H \bE_{s_h \sim \pit} \conftilt(s_h,\pi_t(s_h,h), h)
\end{aligned}\end{equation}
\end{proof}

\subsection{Nice Episodes}\label{subsec:niceepisodes}
The goal in this section is to bound the number of suboptimal episodes.
We use a similar approach as in \cite{dann2017unifying} which uses the concept
of nice episodes but we modify their definition of nice episodes to
to account for the noise the  algorithm adds in order to preserve privacy.
We formally define nice episodes in definition \ref{def:nice}.
The rest of the
proof proceeds by bounding the number of episodes that are not nice
and  bounding the number of nice suboptimal episodes.

Recal that $\ntilt(s,a,h)$ represents the
private count of the number of times state triplet $s,a,h$ has
been visited right before episode $t$. And $\Eeps$ is the error
of the $\eps$-differentially private counter, that is, on any episode $t$
\[
\left|\ntilt(s,a,h)  - \nhatt(s,a,h) \right| < \Eeps
\]
where $\nhatt(s,a,h)$ is the true count.

\begin{definition}[Nice Episodes. Similar to definition 2 in \cite{dann2017unifying}]
\label{def:nice}
Let $w(s,a,h)$ be the probability of visiting state $s$ and taking
action $a$ during episode $t$ and time-step $h$ after following policy $\pit$. An episode $t$ is nice if and only if for all $s,\in \states$,
$a,\in\actions$ and $h\in [H]$  the following two conditions hold:
\begin{align*}
w_t(s,a, h) \leq w_{min} \quad
\vee \quad \frac{1}{4}\sum_{i<t} w_{i}(s,a,h) \geq \ln\frac{SAH}{\beta'} + 2\Eeps
\end{align*}
\end{definition}

\begin{lemma}\label{lem:niceproperties}
	If an episode $t$ is nice, then on $F^c$ for all $s,a,h$ the
	following statement holds
	\begin{align*}
	w_t(s,a,h) < w_{min} \quad\vee\quad  \ntilt(s,a,h) > \frac{1}{4}\sum_{i<t} w_{i}(s,a,h) + \Eeps
	\end{align*}
Plus, it follows that if $w_t(s,a,h)>w_{min}$ then $\ntilt(s,a,h)>2\Eeps$.
\end{lemma}
\begin{proof}
Since we consider the event $F^c$ it holds for all $s,a,h$ triplets
\[
\nhatt(s,a,h) >  \frac{1}{2}\sum_{i<t} w_{i}(s,a,h) - \ln\frac{SAH}{\delta'}
\]
and
\begin{align*}
\ntilt(s,a,h) > \nhatt(s,a,h) - \Eeps
 > \frac{1}{2}\sum_{i<t} w_{i}(s,a,h) - \ln\frac{SAH}{\delta'}- \Eeps
 > \frac{1}{4}\sum_{i<t} w_{i}(s,a,h)
\end{align*}
\end{proof}

\begin{lemma}[Non-nice Episodes. \cite{dann2017unifying}]
\label{lem:notnice}
On the good event $F^c$, the number of episodes that are not
nice is at most
\[
\frac{120S^2AH^4}{\alpha\eps}\polylog(S,A,H, 1/\beta)
\]
\end{lemma}

\begin{lemma}[Nice Episodes rate. \cite{dann2017unifying}]
\label{lem:mainrate}
Let $r\geq 1$ and fix $C>0$ which can depend polynomially on the
relevant quantities
and let $D\geq 1$ which can depend poly-logarithmically on the relevant quantities.
Finally let  $\alpha'>0$ be the target accuracy and
let $\ntilt(s,a,h)$ be the private estimate count with error $\Eeps$.
Then
\begin{align*}
  \sum_{(s,a,h)\notin L_{t}}w_h(s,a,h)
\left(
\frac{C(\log{(T)} + D)}{\ntilt(s,a,h) - \Eeps}
\right)^{1/r}
  \leq \alpha'
\end{align*}
on all but at most
\begin{align*}
\frac{8CSAH^r}{(\alpha')^r}\polylog(T,S,A,H,1/\eps,1/\beta', 1/\alpha')
\end{align*}
nice episodes.
\end{lemma}
\begin{proof}
The proof follows mostly from the argument in
\cite[Lemma E.3]{dann2017unifying}. Let $x\eqdefU (s,a,h)$ denote
a state tuple.
Define the gap in episode $t$ by
\begin{align*}
\Delta_t
=\sum_{\sah\notin L_t}w_t(\sah)\left(\frac{C(\log(T)+D)}{\ntilt(\sah)-\Eeps}\right)^{\tfrac{1}{r}}
=\sum_{\sah\notin L_t}w_t(\sah)^{1-\tfrac{1}{r}}
\left(w_t(\sah)\frac{C(\log(T)+D)}{\ntilt(\sah)-\Eeps}\right)^{\tfrac{1}{r}}
\end{align*}

Using H\"{o}lder's inequality
\begin{align*}
\Delta_t \leq
  \left(\sum_{\sah\notin L_t}C H^{r-1}w_t(\sah)\frac{C(\log(T)+D)}{\ntilt(\sah)-\Eeps}\right)^{\tfrac{1}{r}}
\end{align*}

Now we the properties of nice episodes from lemma \ref{lem:niceproperties}
and the fact that $\sum_{i<t}w_t(\sah) \geq 4\ln(SAH/\delta')\geq 2$.
Then on the good event $F^c$ we have the following bound
\begin{align*}
\ntilt(\sah) \geq \frac{1}{4}\sum_{i<t}w_i(\sah) + \Eeps \geq
\frac{1}{8}\sum_{i\leq t}w_i(\sah) + \Eeps
\end{align*}

The function $\frac{\ln(T) + D}{ x}$ is monotonically decreasing for $x\geq 0$.
Then we bound
\begin{align*}
\Delta_t^r
&\leq\sum_{\sah\notin L_t}C H^{r-1}w_t(\sah)\frac{C(\log(T)+D)}{\ntilt(\sah)-\Eeps}\\
&\leq 8C H^{r-1}\sum_{\sah\notin L_t}w_t(\sah)\frac{(\log(T)+D)}{\sum_{i\leq t}w_i(\sah)}\\
\end{align*}


Let the set of nice episodes be
$N = \left\{ t  : w_t(\sah) < w_{min} \text{ or }
 \frac{1}{4}\sum_{i<t} w_i(\sah) \geq \ln\frac{SAH}{\beta'} + 2\Eeps\right\}$
and define a set $M = \{ t : \Delta_t > \alpha'  \}\cap N$ to be the
set of suboptimal nice epidoes. We know that $|M|\leq T$. Finally
we can bound the total number of suboptimal nice episodes by
\begin{align*}
\sum_{t\in M} \Delta_t^r
&\leq \sum_{t\in M}8C H^{r-1}(\log(T)+D)\sum_{\sah\notin L_t}\frac{w_t(\sah)}{\sum_{i\leq t}w_i(\sah)} \\
&\leq 8C H^{r-1}(\log(T)+D)\sum_{\sah\notin L_t}\sum_{t\in K}\frac{w_t(\sah)}{\sum_{i\leq t}w_i(\sah)\mathbbm{1}(w_i(\sah)>w_{\text{min}})}
\end{align*}

For every $x=(s,a,h)$ consider the sequence $w_i(\sah) \in [w_{\text{min}}, 1]$
with $i \in I = \{ w_i(\sah) \geq  w_{\text{min}}\}$ and apply lemma
\ref{lem:lnbound} to get
\begin{align*}
\sum_{t\in M}\frac{w_t(\sah)}{\sum_{i\leq t}w_i(\sah)\mathbbm{1}(w_i(\sah)>w_{\text{min}})} \leq
 \ln\left(\frac{Te}{w_{\text{min}}} \right)
\end{align*}

Therefore we have
\begin{align*}
\sum_{t\in M} \Delta_t^r
\leq 8C S A H^r(\log(T)+D)\ln\left(\frac{Te}{w_{\text{min}}} \right)
\end{align*}
\end{proof} 

Since each episode has to contribute at least $(\alpha')^r$ to this bound
we have
\begin{align*}
|M| \leq \frac{8C S A H^r(\log(T)+D)\ln\left(\frac{Te}{w_{\text{min}}} \right)}
{(\alpha')^2}
\end{align*}

Completing the proof.
\begin{lemma}{\cite[Lemma E.5]{dann2017unifying}}
\label{lem:lnbound}
Let $a_i$ be a sequence taking values in $[a_{min}, 1]$ with $a_{min}>0$
and $m>0$, then  $\sum_{k=1}^m \frac{a_i}{\sum_{i=1}^k a_i}\leq \ln(\frac{me}{a_{min}})$
\end{lemma}

%
%
\section{PAC and Regret Lower Bound Proofs}

\subsection{PAC Lower Bound. Proof of theorem \ref{thm:lowerbound}}
In this section we provide the analysis of PAC lower bound from
section \ref{sec:lower}. Below is the proof of theorem \ref{thm:lowerbound}.

\begin{theorem*}[PAC Lower Bound. Theorem \ref{thm:lowerbound}]
Let $\algo$ be an RL agent satisfying $\eps$-JDP.
Suppose that $\algo$ is $(\alpha, \beta)$-PAC for some $\beta \in (0,1/8)$.
Then, there exists a fixed-horizon episodic MDP where the number of episodes until the algorithm's policy is $\alpha$-optimal with probability at least $1 - \beta$ satisfies
\begin{align*}
\Ex{n_\algo} \geq \Omega\left( \frac{SAH^2}{\alpha^2} + \frac{SAH}{\alpha\epsilon}
\ln\left(\frac{1}{\beta}\right)\right) \enspace.
\end{align*}
\end{theorem*}

\begin{proof}[Proof of Theorem~\ref{thm:lowerbound}.]

The proof follows five main steps:
1) We consider the easier case of JDP with public-initial-state setting, (see \cref{def:jdppublic}) and a class of hard-MDPs (\cref{fig:hardMDP}).
2) In \cref{lem:privMAB}, we give a sample complexity lower bound of differentially-private best-arm-identification for MAB.
3) In \cref{lem:jdptodp}, we show that learning the MDP with JDP in the public initial state setting is the same as learning $S$ best-arm-identification MAB instances with differential privacy (DP).
4) We use \cref{lem:privMAB} and \cref{lem:jdptodp} to get \cref{lem:lowerboundpublic}, which gives a lower bound for any RL agent with JDP in the public-initial-state setting.
5) Finally, \cref{lem:publicvsprivatestates} shows that given that class of hard MDPs, any agent satisfying JDP in the public-initial-state setting also satisfies JDP. Therefore the lower bound in \cref{lem:lowerboundpublic} applies, and that concludes the proof of \cref{thm:lowerbound}.
\end{proof}

\subsection{Proofs from Section~\ref{sec:mablb}}

The lower bound result relies on the following adaptation of the coupling lemma from \citep[Lemma 6.2]{karwa2017finite}.
\begin{lemma}[\cite{karwa2017finite}]\label{lemma:KarwaVadhan}
For every pair of distributions $\mathbb{D}_{\theta_0}$ and $\mathbb{D}_{\theta_1}$, every $(\epsilon, \delta)$-differentially private mechanism $M(x_1, \ldots, x_n)$, if $\mathbb{M}_{\theta_0}$ and $\mathbb{M}_{\theta_1}$ are two induced marginal distributions on the output of $M$ evaluated on input dataset $X_1, \ldots, X_n$ sampled i.i.d from $\mathbb{D}_{\theta_0}$ and $\mathbb{D}_{\theta_1}$ respectively, $\epsilon' = 6\epsilon n \|\mathbb{D}_{\theta_0}-\mathbb{D}_{\theta_1}\|_{tv}$ and $\delta' = 4\delta e^{\epsilon'}\|\mathbb{D}_{\theta_0}-\mathbb{D}_{\theta_1}\|_{tv}$, then, for every event $E$,
\begin{align*}
 \mathbb{M}_{\theta_0}(E) \leq e^{\epsilon'}\mathbb{M}_{\theta_1}(E) + \delta'
\end{align*}
\end{lemma}

\begin{lemma*}[Lemma \ref{lem:KarwaVadhanMAB}.]
Fix any arm $a \in [k]$. Now consider any pair of MAB instances $\mu, \nu \in [0,1]^k$ both with $k$ arms and time horizon $T$, such that $\|\mu_a - \nu_a \|_{tv} < \alpha$ and  $\|\mu_{a'} - \nu_{a'} \|_{tv} = 0$ for all $a' \neq a$. Let $R \sim \bernoulli(\mu)^T$ and $Q \sim \bernoulli(\nu)^T$ be the sequence of $T$ rounds of rewards sampled under $\mu$ and $\nu$ respectively, and let $\algo$ be any $\epsilon$-DP multi-armed bandit algorithm. Then, for any event $E$ such that under event $E$ arm $a$ is pulled less than $t$ times,
\begin{align*}
\prob{\algo,R}{E} \leq e^{6 \epsilon t \alpha}\prob{\algo,Q}{E}
\end{align*}
\end{lemma*}
\begin{proof}
We can think of algorithm $M(R)$ as  taking as input a tape of $t$ pre-generated rewards for arm $a$, denote this tape as $R = (r_1, \ldots r_t)$. If $M(R)$ is executed with input tape $R$, then when $M(R)$ pulls arm $a$ for the $j^{th}$ time the $j^{th}$ entry $R_j$ is revealed and removed from $R$. If $M(R)$ runs out of the tape $R$ then the reward is drawn from the real distribution of arm $a$ (i.e. $P_a$). Lastly, if $M(R)$ pulls some arm $a' \neq a$ then the reward is drawn from the real distribution $P_{a'}$.

Note that if $R$ is sampled from the real distribution of arm $a$ i.e $R \sim P_a^t$, then $M$ and $M(R)$ are equivalent. That is, for any event $E$,
$$\prob{M,P}{E} = \prob{M(R), P}{E}$$
%
%

Under this construction, the event that $M(R)$ pulls arm $a$ less than $t$ times, is the same as the event that $M(R)$ consumes less than $t$ entries of the tape $R$. By the assumption of the event $E$ under consideration, if $M(R)$ consumes at least $t$ entries of the tape then we can say that event $E$ fails to happen. Therefore, in order to evaluate the event $E$ we only need to initialize $M(R)$ with tapes of size $t$. Furthermore, we treat the input tape $R$ as the data of $M$ and we claim that $M(R)$ is ($\epsilon, \delta$)-differentially private on $R$.

Now we apply lemma (\ref{lemma:KarwaVadhan}) to bound the probability of $E$ under $M(R_p)$ and $M(R_q)$, where $R_p$ and $R_q$ are the input tapes each generated with $t$ i.i.d samples from distribution $P_a$ and $Q_a$ respectively.
$$\prob{M(R_p),P}{E} \leq e^{6 \epsilon t \Delta_a} \prob{M(R_q),Q}{E}$$
This implies $\prob{M,P}{E} \leq e^{6 \epsilon t \Delta_a}\prob{M,Q}{E}$.
\end{proof}
\subsection{Proofs from Section~\ref{sec:rlpslb}}
%
%
\begin{lemma*}[Lemma \ref{lem:jdptodp}.]
Let $(U,S_1)$ be a user-state input sequence with initial states from some set $\states_1$.
Suppose $\cM$ is an RL agent that satisfies $\eps$-JDP in the public initial state setting.
Then, for any $s \in \states_1$ the trace $\algo_{1,s}(U,S_1)$ is the output of an $\epsilon$-DP MAB mechanism on input $U_s$.
\end{lemma*}
\begin{proof}[Proof of Lemma~\ref{lem:jdptodp}]

 \newcommand{\Eat}{\bar E_{-t}^{|\vec{a}} }
 \newcommand{\EaG}{\bar E_{>t}^{|\vec{a}} }
 \newcommand{\EaL}{\bar E_{<t}^{|\vec{a}} }

Fix $(U,S_1)$, $s$ and $U_s$ as in the statement.
Recall that $T_s$ is the number of times state $s$ is in $S_1$.
Observe that $\algo_{1,s}(U,S_1)$ has the output type expected from a MAB mechanism on input $U_s$.
Fix an event $E \subseteq \actions^{T_s+1}$
on the first action from all episodes starting with $s$ together with the
action predicted by the policy at state $s$.
For any $\bar{a} = (\bar a^{(1)},\ldots,\bar a^{(T_s)},\hat{a}) \in E$ we
define the
event $E_{\bar{a}} \subseteq \actions^{H \times T} \times \Pi$  by
\begin{align*}
E_{\bar{a}} = \{ (a_h^t)_{h\in[H], t\in [T]}, \pi | a_1^{(t_1)} = \bar a^{(1)}, \ldots, a_1^{(t_{T_s})} = \bar a^{(T_s)}, \pi(s) = \hat{a}\} \enspace.
\end{align*}
where $a_1^{t_i}$ is the first action in the $i$th episode where state $s$ is the first state.
The the event $\bar E$ is the union of all events $E_{\bar a}$, defined as
\begin{align*}
\bar{E} = \cup_{\bar{a} \in E} E_{\bar{a}} \subseteq \actions^{H \times T} \times \Pi
\end{align*}

Let $\bar{E}_{-t} \subseteq \actions^{H \times (T-1)}\times \Pi$ be the collection of outputs from $\bar{E}$ truncated to length $T-1$
and including the output policy. 
%
%
Furthermore, let $\bar{E}_{< t} \subseteq \actions^{H \times [t-1]}$ be the collection of outputs from $\bar{E}$ truncated to length $t-1$ and similarly let $\bar{E}_{> t} \subseteq \actions^{H \times (T - t)}$ be the sequences truncated to length $T-t$ . For any $\vec{a} \in \actions^{H}$ we define the following notation
\begin{align*}
&\Eat=\left\{e \in \bar{E}_{-t} : (\vec{a},e) \in \bar{E} \right\} \\
&\EaG = \left\{e \in \bar{E}_{>t}:
\exists_{ {b}^{(1)},\ldots,{b}^{(t-1)}\in \actions^{H}}
\left( {b}^{(1)},\ldots,{b}^{(t-1)}, \vec{a},e\right) \in \bar{E}
\right\} \\
&\EaL = \left\{e \in \bar{E}_{<t}:
\exists_{ {b}^{(1)},\ldots,{b}^{(T-t)}\in \actions^{H}}
\left({b}^{(1)},\ldots,{b}^{(T-t)}, \vec{a},e\right) \in \bar{E}  \right\}
\enspace.
\end{align*}
For the remaining of the proof, denote by $\algo_t(U,S_1)$ the output during episode $t$, $\algo_{<t}(U,S_1)$ all the outputs before episode $t$,
 $\algo_{>t}(U,S_1)$ all the outputs after episode $t$, and $\algo_{-t}(U,S_1)$
are all the outputs except for the output during episode $t$ and it includes the
final output policy.
For any $\vec{a}\in \actions^H$ It is easy to show that
\begin{align}\label{eq:Eequiv}
  \text{ }  \algo_{-t}(U,S_1) \in \Eat \text { if and only if }
  \algo_{>t}(U,S_1) \in \EaG \text{ and } \algo_{<t}(U,S_1) \in \EaL
\end{align}
Observe that since $\algo$ processes its inputs incrementally we have
that
\begin{align}
\label{eq:futureconditioned}
\pr{\algo_{t}(U,S_1) = \vec{a} \land \algo_{<t}(U,S_1) \in  \EaL
~\bigg| \algo_{>t}(U,S_1) \in  \EaG} =
\pr{\algo_{t}(U,S_1) = \vec{a} \land \algo_{<t}(U,S_1) \in  \EaL}
\end{align}
The equation \eqref{eq:futureconditioned} says that conditioning on the output
of future events does not affect the probability of the present event.

Now take $(U',S_1)$ to be a $t$-neighboring user-state sequence and note $U'_s$ is a neighboring sequence of $U_s$ in the sense used in the definition of DP for MAB mechanisms.
The next equation says that the output of $\algo_t$ on episode $t$, is not distinguishable
on the user-state sequences $(U, S_1)$ and $(U', S_1)$. This is because $U$ and $U'$ match on
all episodes before episode $t$ and they share the same initial state on every episode.
We have that
\begin{align}
\label{eq:neighborevent}
\pr{\algo_t(U, S_1)} = \pr{\algo_t(U', S_1)}
\end{align}

We will use the following simple application of Baye's Rule:
\begin{align}
  \label{eq:bayes}
  \pr{A | B,C } = \frac{\pr{A \land B | C}}{\pr{B|C}}
\end{align}

Next We want to show that
\begin{align}\label{eq:UtoUprime}
  \pr{\algo_{t}(U,S_1) = \vec{a} ~\bigg| \algo_{-t}(U,S_1) \in  \Eat}
   = \pr{\algo_{t}(U',S_1) = \vec{a} ~\bigg| \algo_{-t}(U',S_1) \in  \Eat}
\end{align}
Let fix one $\vec{a}$ for now, then from \eqref{eq:Eequiv}, \eqref{eq:futureconditioned}, \eqref{eq:neighborevent}, and \eqref{eq:bayes} we have
\begin{align*}
&\pr{\algo_{t}(U,S_1) = \vec{a} ~\bigg| \algo_{-t}(U,S_1) \in  \Eat} \\
&= \pr{\algo_{t}(U,S_1) = \vec{a} ~\bigg| \algo_{>t}(U,S_1) \in \EaG
\land  \algo_{<t}(U,S_1) \in \EaL}
\quad \text{equation } \eqref{eq:Eequiv}\\
&= \frac{
\pr{\algo_{t}(U,S_1) = \vec{a} \land \algo_{<t}(U,S_1) \in \EaL ~\bigg| \algo_{>t}(U,S_1) \in \EaG}}
{\pr{\algo_{<t}(U,S_1) \in \EaL ~\bigg| \algo_{>t}(U,S_1) \in \EaG}}
\quad \text{equation } \eqref{eq:bayes} \\
&= \frac{
\pr{\algo_{t}(U,S_1) = \vec{a} \land \algo_{<t}(U,S_1) \in \EaL }}
{\pr{\algo_{<t}(U,S_1) \in \EaL }}
\quad \text{equation } \eqref{eq:futureconditioned} \\
&= \frac{
\pr{\algo_{t}(U',S_1) = \vec{a} \land \algo_{<t}(U',S_1) \in \EaL }}
{\pr{\algo_{<t}(U',S_1) \in \EaL }}
\quad \text{equation } \eqref{eq:neighborevent} \\
&= \frac{
\pr{\algo_{t}(U',S_1) = \vec{a} \land \algo_{<t}(U',S_1) \in \EaL ~\bigg| \algo_{>t}(U',S_1) \in \EaG}}
{\pr{\algo_{<t}(U',S_1) \in \EaL ~\bigg| \algo_{>t}(U',S_1) \in \EaG}}
\quad \text{equation } \eqref{eq:futureconditioned} \\
&= \pr{\algo_{t}(U',S_1) = \vec{a} ~\bigg| \algo_{>t}(U',S_1) \in \EaG
\land  \algo_{<t}(U',S_1) \in \EaL}
\quad \text{equation } \eqref{eq:bayes}\\
&=\pr{\algo_{t}(U',S_1) = \vec{a} ~\bigg| \algo_{-t}(U',S_1) \in  \Eat}
\end{align*}

Combined with the $\epsilon$-JDP assumption on $\algo$ this implies that
\begin{align*}
& \pr{\algo(U,S_1) \in \bar{E}} \\
&=
\sum_{\vec{a} \in \actions^H}  \pr{\algo_t(U,S_1) = \vec{a}
\land \algo_{-t}(U, S_1) \in \Eat} \\
&=
\sum_{\vec{a} \in \actions^H} \pr{\algo_{t}(U,S_1) = \vec{a} ~\bigg| \algo_{- t}(U,S_1) \in  \Eat} \pr{\algo_{- t}(U,S_1) \in  \Eat} \\
&=
\sum_{\vec{a} \in \actions^H} \pr{\algo_{t}(U',S_1) = \vec{a} ~\bigg| \algo_{- t}(U',S_1) \in  \Eat} \pr{\algo_{- t}(U,S_1) \in  \Eat}
\quad\quad \hfill \text{ Equation } \eqref{eq:UtoUprime} \\
&\leq
\sum_{\vec{a} \in \actions^H} \pr{\algo_{t}(U',S_1) = \vec{a} ~\bigg| \algo_{- t}(U',S_1) \in  \Eat} e^{\eps}\pr{\algo_{-t}(U',S_1) \in  \Eat}
\quad\quad \hfill \algo \text{ satisfies  } \eps\text{-JDP}  \\
&=
e^{\eps}
\pr{\algo(U',S_1) \in \bar{E}}
%
\enspace.
\end{align*}
Finally, using the inequality above, and by the construction of $\algo_{1,s}$ and $\bar{E}$ we have
\begin{align*}
\pr{\algo_{1,s}(U,S_1) \in E}
&=
\pr{\algo(U,S_1) \in \bar{E}} \\
&\leq
e^{\epsilon} \pr{\algo(U',S_1) \in \bar{E}} \\
&= e^{\epsilon}\pr{\algo_{1,s}(U',S_1) \in E} \enspace.
\end{align*}
\end{proof}

To prove the lower bound we consider the class of MDPs shown in Figure~\ref{fig:hardMDP}.
An MDP in this class has state space $\states \coloneqq [n] \cup \{+,-\}$ and action space $\actions \coloneqq \{0,\ldots,m\}$.
On each episode, the agent starts on one of the initial states
 $\{1, \ldots, n\}$ chosen uniformly at random.
The state labelled $0$ is a dummy state which represents the initial
transition to any state $s \in \{1, \ldots, n\}$ with uniform probability.
On each of the initial states the agent has $m+1$ possible actions and transitions can only take it to one of two possible absorbing states $\{+,-\}$.
Lastly, if the current state is either one of $\{ +, - \}$ then the only possible transition is
a self loop, hence the agent is stays in that state until the
end of the episode. We assume in these absorbing states the agent can only take a fixed action.
Every action which transitions to state $+$ provides reward $1$ while actions transitioning to state $-$ provide reward $0$.
In particular, in each episode the agent either receives reward $H$ or $0$.

Such an MDP can be seen as consisting of $n$ parallel MAB problems.
Each MAB problem determines the transition probabilities between the initial state $s \in\{1, \ldots, n\}$ and the absorbing states $\{+,-\}$.
We index the possible MAB problems in each initial state by their optimal arm, which is always one of $\{0,\ldots,m\}$.
We write $I_s \in \{0,\ldots,m\}$ to denote the MAB instance in initial state $s$, and define the transition probabilities such that $\pr{+|s,0} = 1/2+\alpha'/2$ and $\pr{+|s,a'} = 1/2$ for $a' \neq I_s$ for all $I_s$, and for $I_s \neq 0$ we also have $\pr{+|s,I_s} = 1/2 + \alpha'$.
Here $\alpha'$ is a free parameter to be determined later.
We succinctly represent an MDP in the class by identifying the optimal action (i.e.\ arm) in each initial state: $I \coloneqq (I_1,\ldots,I_n)$.

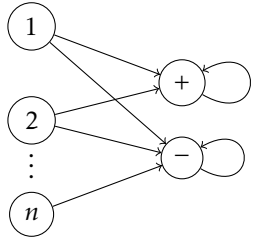
\begin{figure}[h]
\begin{equation*}
  \begin{tikzpicture}
	\begin{pgfonlayer}{nodelayer}
		\node [style=States] (1) at (0, 2.5) {$1$};
		\node [style=States] (2) at (0, 0) {$2$};
		\node [style=States] (3) at (0, -2.5) {$n$};
		\node [style=States] (4) at (4, 1) {$+$};
		\node [style=States] (5) at (4, -1) {$-$};
		\node [style=none] (7) at (0, -1) {$\vdots$};
	\end{pgfonlayer}
	\begin{pgfonlayer}{edgelayer}
		\draw [style=arrow] (1) to (4);
		\draw [style=arrow] (1) to (5);
		\draw [style=arrow] (2) to (4);
		\draw [style=arrow] (2) to (5);
		\draw [style=arrow] (3) to (5);
		\draw [style=arrow, in=30, out=-30, loop] (4) to ();
		\draw [style=arrow, in=30, out=-30, loop] (5) to ();
	\end{pgfonlayer}
\end{tikzpicture}
\end{equation*}
\caption{Hard MDP}\label{fig:hardMDP}
\end{figure}

\begin{proof}[Proof of Lemma~\ref{lem:lowerboundpublic}]
We start by noting that the first term in the lower bound comes from the corresponding lower bound for the non-private episodic RL setting \citep[Theorem 2]{dann2015sample}, which also holds for our case.

Now let $I = (I_1,\ldots,I_n)$ encode an MDP from the class above with $n+2$ states and $m+1$ actions.
The optimal policy on this MDP is given by $\pi^*(s) = I_s$ for $s \in [n]$, and we write $\rho_I^*$ to denote the total expected reward of the optimal policy on a single episode.

Define $G_s$ to be the event that policy $\pi$ produced by algorithm $\algo$ finds the optimal arm in state $s$, that is $\pi(s)=I_s$.
We denote by $\rho_I^{\pi}$ the total expected reward per episode of this policy.
Then, for any episode, the difference $\rho_I^*-\rho_I^\pi$ between total rewards is at least
$$
\rho_I^*-\rho_I^\pi \geq H \left(1 - \frac{1}{n}\sum_{s=1}^n \mathbbm{1}\{G_s\} \right) \alpha'/2 \enspace.
$$
Thus, $\pi$ cannot by $\alpha$-optimal unless we have:
\begin{align}
&\alpha \geq H \left(1 - \frac{1}{n}\sum_{s=1}^n \mathbbm{1}\{G_s\} \right) \alpha'/2 \notag \\
\iff &\frac{2\alpha}{H \alpha'} \geq\left(1 - \frac{1}{n}\sum_{s=1}^n \mathbbm{1}\{G_s\} \right) \notag\\
\iff &\frac{1}{n}\sum_{i=1}^n\mathbbm{1}\{G_s\} \geq \left( 1  - \frac{2\alpha}{H \alpha'} \right)\coloneqq \phi \enspace. \label{eq:epsOptFrac}
\end{align}
Here choose $\phi = 6/7$ and set $\alpha' = \frac{14 \alpha}{H}$.
Equation \eqref{eq:epsOptFrac} says that in order to make $\pi$ an $\alpha$-optimal policy we must solve at least a $\phi$ fraction of the MAB instances.

Hence, to get an $\alpha$-optimal with probability at least $1-\beta$ we require
\begin{equation*}
1-\beta \leq \prob{I}{\rho_I^*-\rho_I^\pi \leq \alpha} \leq \prob{I}{\frac{1}{n}\sum_{s=1}^n\mathbbm{1}\{G_s\} \geq \phi} \enspace,
\end{equation*}
and by Markov's inequality we have
\begin{align*}
\prob{I}{\frac{1}{n}\sum_{s=1}^n\mathbbm{1}\{G_s\} \geq \phi} \leq \frac{1}{n \phi}\sum_{s=1}^n\prob{I}{G_s}
\enspace.
\end{align*}
Each $\mathbbm{1}\{G_s\}$ is independent from each other be construction of the MDP. Now letting $\beta_s$ be an upper bound for the fail probability of each $\{G_s\}$, the derivation above implies that $1-\beta \leq \frac{1}{n \phi}\sum_{s=1}^n (1 - \beta_s)$, or, equivalently, that $\sum_{s} \beta_s \leq n (1 - \phi (1 - \beta))$.

Now note that Lemma~\ref{lem:jdptodp} implies that all interactions between $\algo$ and $I$ that start on state $s$ constitute the execution of an $(\epsilon,\delta)$-DP algorithm on the MAB instance at state $s$.
Hence, by Lemma~\ref{lem:privMAB} we can only have $\prob{I}{G_s} \geq 1 - \beta_s$ for some $\beta_s < 1/4$ if the number of episodes starting at $s$ where $\algo$ chooses an $\alpha'$-suboptimal arm satisfies
\begin{align*}
\Ex{n_s} &> \frac{(A-1)}{24 \epsilon \alpha'}\ln{\left(\frac{1}{4\beta_s}\right)}\mathbbm{1}[\beta_s < 1/4] \\
&= \frac{H(A-1)}{336 \epsilon \alpha}\ln{\left(\frac{1}{4\beta_s}\right)}\mathbbm{1}[\beta_s < 1/4] \\
&\geq
\frac{H(A-1)}{336 \epsilon \alpha} \ln{\left(\frac{1}{4\beta_s}\right)}\mathbbm{1}[\beta_s \leq 1 - \phi(1-\beta)]
\enspace,
\end{align*}
where we used that $\phi = 6/7$ and $\beta < 1/8$ imply $1 - \phi(1-\beta) < 1/4$, and that each MAB instance has $A - 1$ arms which are $\alpha'$-suboptimal.


Thus, we can find a lower bound $\Ex{n_{\algo}} \geq \sum_s \Ex{n_s}$ by minimizing the sum of the lower bound on $\Ex{n_s}$ under the constraint that $\sum_{s} \beta_s \leq n (1 - \phi (1 - \beta))$.
Here we can apply the argument from \citep[Lemma D.1]{dann2015sample} to see that the optimal choice of probabilities is given by $\beta_s = 1 - \phi(1-\beta)$ for all $s$.
Plugging this choice in the lower bound leads to
\begin{align*}
\Ex{n_{\algo}}
&\geq
\frac{H S (A-1)}{336 \epsilon \alpha} \ln{\left(\frac{7}{4+24\beta}\right)}
\enspace.
\end{align*}

%
%
\end{proof}

\begin{lemma*}[Lemma \ref{lem:publicvsprivatestates}]
Any RL agent $\algo$ satisfying $\eps$-JDP also satisfies $\eps$-JDP in the public state setting.
\end{lemma*}
\begin{proof}
Suppose that algorithm $\algo$ satisfies $\eps$-JDP.
Let $(U, S_1)$ and $(U', S_1')$ be two $t$-neighboring user-state
sequences such that $S_1 = S_1'$. Then for all events
$E \subseteq \actions^{H \times [T-1]} \times \Pi$ we have
\begin{align*}
\pr{\algo_{-t}(U,S_1) \in E } \leq e^\eps \pr{\algo_{-t}(U', S_1') \in E}
\end{align*}
Therefore $\algo$ satisfies the condition for $\eps$-JDP in the
public state setting as in definition \ref{def:jdppublic}.
\end{proof}


\subsection{Regret Lower Bound. Proof of theorem \ref{thm:regretlower}}
In this section we provide the complete lower bound regret
analysis of algorithm \PUCB~ from theorem \ref{thm:regretlower}. We
restate the argument here:
\begin{theorem*}[Private Regret Lower Bound. Theorem \ref{thm:regretlower}]
  For any $\eps$ JDP-algorithm $\algo$ there exist an MDP $M$ with $S$
  states $A$ actions over $H$ time steps per episode such that the
  expected regret after $T$ steps is
\begin{align*}
\Ex{\regret(T)} = {\Omega}
\left (\sqrt{HSA T} + \frac{S A H\log(T)}{\eps}\right)
\end{align*}
for any $T \geq S^{1.1}$.
\end{theorem*}
\begin{proof}{of theorem \ref{thm:regretlower}}
The first term in the bound comes from the non-private regret
due \cite{jaksch2010near}, which states that the expected regret
is lower-bounded by
\[
\Omega\left( \sqrt{HSAT}\right)
\]

Next, we analyze the regret lower bound due to privacy. Like section \ref{sec:rlpslb}, we first consider the regret lower bound of any $ \eps $-differentially private algorithm under the public-initial-state setting. We also utilize the same construction of hard MDP instances, as depicted in figure \ref{fig:hardMDP}.

Let $\algo$ be an RL agent and $(U, S_1)$ a user-state input sequence with initial state from some set $S_1$. Let $\algo(U,S_1) = (\vec{a}^{(1)},\ldots,\vec{a}^{(T)},\pi) \in \actions^{H \times T} \times \Pi$ be the collection of all outputs produced by the agent on inputs $U$ and $S_1$. For every $s \in \states_1$ we write $\algo_{1,s}(U,S_1)$ to denote the restriction of the previous trace to contain just the first action from all episodes starting with $s$ together with the action predicted by the policy at states $s$:
\begin{align*}
\algo_{1,s}(U,S_1) \coloneqq \left(a_1^{(t_{s,1})}, \ldots, a_1^{(t_{s,T_s})}, \pi(s)\right) \enspace,
\end{align*}
where $T_s$ is the number of occurrences of $s$ in $S_1$ and $t_{s,1}, \ldots, t_{s,T_s}$ are the indices of these occurrences. Furthermore, given $s \in \states_1$ we write $U_s = (u_{t_{s,1}},\ldots,u_{t_{s,T_s}})$ to denote the set of users whose initial state equals $s$. Then from lemma \ref{lem:jdptodp} we have that the trace $\algo_{1,s}(u,s_1)$ is the output of a MAB algorithm satisfying $\eps$-DP.

Thus, we have reduced the problem to learning $n=S-2$ MAB instances satisfying $\eps$-DP where each MAB instances is visited $T_s$ many times, for all $s \in [n-2]$.
Now we can use the result from \cite{shariff2018differentially}
which states that the regret  of any $\eps$-DP algorithm for
the MAB problem with $A$ arms is lower bounded by
$\Omega\left( \frac{A \log(T)}{ \eps} \right)$
where $T$ is the total number of arm pulls.
By our MDP construction, a state is selected uniformly at random
at the beginning of the episode. Then the learner takes a single action
and receives a reward in $\{0, H-1\}$, for this reason the regret of
each MAB learner is scaled by $H$ in our setting.

Hence, for each initial state $s\in \{1,\ldots, n\}$, the trace $\algo_{1,s}(u,s_1)$ produces a sequence of actions satisfying $\eps$-DP and with regret at least
$\Omega\left( \frac{A \log(T_s)}{ \eps} \right)$. Combining the regret corresponding to each initial state $s\in \{1,\ldots, n\}$, the regret of the agent must be at least
\begin{align*}
\Omega\left(  \frac{AH }{ \eps}\sum_{s\in\states}\log(T_s)\right)
\end{align*}
where $T_s$ is a random variable.
Next we use the Markov inequality to lower bound the
term $\sum_{s\in\states}\log(T_s)$ by
\[
\sum_{s\in\states}\log(T_s) = S \Ex{\log(T_s)} \geq
S\log(\tfrac{T}{S}) \pr{\log(T_s)\geq \log(\tfrac{T}{S})}
\]
The event $ \log(T_s)\geq \log(\tfrac{T}{S}) $ happens only
when $T_s \geq \tfrac{T}{S}$.
Since each $s\in\{1,\ldots, n\}$ is selected with equal probability
at the beggining of the episodes,
in expectation  the number of pulls is $\Ex{T_s} = \tfrac{T}{S}$.
Thus, each random variable $T_s$ follows a binomial distribution
with mean $\tfrac{T}{S}$ therefore the probability that $T_s \geq \tfrac{T}{S}$
is $\tfrac{1}{2}$. Replacing the probability term we get that
%
the total regret of the RL algorithm is lower bounded by:

\begin{align*}
  \Ex{\regret(T)} = \Omega\left(
\frac{A H S\log(\tfrac{T}{S})}{\eps}
  \right)
\end{align*}
\end{proof}

\end{document}